\documentclass{vldb}
\usepackage{times}
\usepackage{graphicx}
\usepackage{epstopdf}
\usepackage{amsmath,epsfig}
\usepackage{subfigure}
\usepackage{enumitem}
\usepackage[labelfont=bf,textfont=bf]{caption}

\usepackage[vlined,boxruled,linesnumbered]{algorithm2e}
\SetAlFnt{\scriptsize}
\SetAlCapFnt{\scriptsize}
\SetAlCapNameFnt{\scriptsize}

\makeatletter
\renewcommand{\algocf@caption@boxruled}{%
  \hrule
  \hbox to \hsize{%
    \vrule\hskip-0.4pt
    \vbox{
       \vskip\interspacetitleboxruled%
       \unhbox\algocf@capbox\hfill
       \vskip\interspacetitleboxruled
       }%
     \hskip-0.4pt\vrule%
   }\nointerlineskip%
}%
\makeatother

\usepackage{listings}
\usepackage{framed}
\usepackage{lipsum}
\usepackage{caption}
\usepackage{fancybox}
\usepackage{balance}

\usepackage{cases}
\usepackage{algpseudocode}
\graphicspath{{./figures/}}

\usepackage[T1]{fontenc}

\usepackage{xcolor}

\newcommand{\reviseyq}[1]{#1} 


\newcommand{\nop}[1]{}

\let\oldnl\nl
\newcommand{\nonl}{\renewcommand{\nl}{\let\nl\oldnl}}

\newtheorem{theorem}{Theorem}

\newtheorem{definition}{Definition}
\newtheorem{example}{Example}

\newtheorem{lemma}{Lemma}

\newtheorem{pdef}{Problem Definition}

\DeclareMathOperator*{\argmax}{arg\,max}

\begin{document}
\sloppy

\title{KBQA: Learning Question Answering over QA Corpora and Knowledge Bases}

\nop{
\numberofauthors{6}
\author{
\alignauthor
Wanyun Cui\\
\affaddr{Shanghai Key Laboratory of Data Science\\School of Computer Science, Fudan University}
       \email{wanyuncui1@gmail.com}
\alignauthor
Yanghua Xiao\thanks{\small Correspondence author. This paper was supported by the National
Key Basic Research Program of China under No.2015CB358800,
by the National NSFC (No.61472085, U1509213), by Shanghai
Municipal Science and Technology Commission foundation key
project under No.15JC1400900, by Shanghai Municipal Science
and Technology project under No.16511102102.}\\
       \affaddr{Shanghai Key Laboratory of Data Science\\School of Computer Science, Fudan University}
       \email{wanyuncui1@gmail.com}
\alignauthor
Haixun Wang\\
       \affaddr{Google Research}\\
       \email{haixun@google.com}
\and
\alignauthor
Yangqiu Song\\
       \affaddr{UIUC}\\
       \email{yqsong@illinois.edu}
       \alignauthor
Seungwon Hwang\\
       \affaddr{POSTECH}\\
       \email{swhwang@postech.ac.kr}
\alignauthor
Wei Wang\\
       \affaddr{Fudan University}\\
     \email{weiwang1@fudan.edu.cn}
}
}

\numberofauthors{1}
 \author{\alignauthor Wanyun Cui$^{\S}$ \; Yanghua Xiao$^{\S}
 $ \;   Haixun Wang$^{\ddag}$ \; Yangqiu Song$^{\P}$   \; Seung-won Hwang$^{\sharp}$
 \;  Wei Wang$^{\S}$\;   \\
   \affaddr {$^{\S}$Shanghai Key Laboratory of Data Science, School of Computer Science, Fudan University\\
   \affaddr {$^{\ddag}$Facebook \; $^{\P}$HKUST \;  $^{\sharp}$Yonsei University} \\
   \affaddr {wanyuncui1@gmail.com, shawyh@fudan.edu.cn, haixun@gmail.com, yqsong@cse.ust.hk, seungwonh@yonsei.ac.kr, weiwang1@fudan.edu.cn}
}}

\nop{
\author{ Wanyun Cui, Yanghua Xiao, \thanks{\small Correspondence author. This paper was supported by the National
Key Basic Research Program of China under No.2015CB358800,
by the National NSFC (No.61472085, U1509213), by Shanghai
Municipal Science and Technology Commission foundation key
project under No.15JC1400900, by Shanghai Municipal Science
and Technology project under No.16511102102.} \\
Shanghai Key Laboratory of Data Science, School of Computer Science, Fudan University\\
wanyuncui1@gmail.com, shawyh@fudan.edu.cn
\AND
Haixun Wang\\
Facebook, USA \\
haixun@gmail.com
\And
Seung-won Hwang\\
Yonsei University\\
seungwonh@yonsei.ac.kr
\And
Wei Wang\\
Shanghai Key Laboratory of Data Science\\
School of Computer Science, Fudan Uni.\\
weiwang1@fudan.edu.cn
}
}

\maketitle
\begin{abstract}
Question answering (QA) has become a popular way for humans to access billion-scale
knowledge bases. Unlike web search, QA
over a knowledge base gives out accurate and concise results, provided
that natural language questions can be understood and mapped
precisely to structured queries over the knowledge base.  The
challenge, however, is that a human can ask one question in
many different ways. Previous approaches have natural limits due to their representations: rule based approaches only understand a
small set of ``canned'' questions, while keyword based or synonym based approaches cannot fully understand the questions.
 In this paper, we design a new kind of question representation: \textbf{templates}, over a billion scale knowledge base and a million scale QA corpora.  For
example, for questions about a city's population, we learn templates such as {\tt What's the population of \$city?}, {\tt How many
people are there in \$city?}. We learned 27 million templates for 2782 intents. Based on these templates, our QA system KBQA
effectively supports binary factoid questions, as well as complex questions which are composed of a series of binary factoid questions.
Furthermore, we expand
predicates in RDF knowledge base, which boosts the coverage of
knowledge base by 57 times. Our QA system beats all other state-of-art works on both effectiveness and efficiency over QALD benchmarks.
\end{abstract}

\section{Introduction}
\label{sec:intro}

Question Answering (QA) has drawn a lot of research interests. A QA
system is designed to answer a particular type of
questions~\cite{burger2001issues}. One of the most important types of
questions is the factoid question (FQ), which asks about objective
facts of an entity.  A particular type of FQ, known as the {binary
  factoid question (BFQ)}~\cite{agichtein2005analysis}, asks about a property
of an entity. For example, {\tt how many people are there in
  Honolulu?}  If we can answer BFQs, then we will be able to
answer other types of questions, such as 1) ranking questions: {\tt
  which city has the 3rd largest population?}; 2) comparison
questions: {\tt which city has more people, Honolulu or New
  Jersey?}; 3) listing questions: {\tt list cities ordered
  by population} etc. In addition to BFQ and its variants, we can answer a complex factoid question such as {\tt when was Barack
  Obama's wife born?} This can be answered by combining the answers of two BFQs: {\tt who's Barack Obama's wife?} (Michelle Obama) and {\tt when
  was Michelle Obama born?} (1964). We define a complex factoid question as a question
that can be decomposed into a series of BFQs. In
this paper, we focus on BFQs and complex factoid questions.





QA over a knowledge base has a long history.
In recent years, large scale knowledge bases become available,
including Google's Knowledge Graph,
Freebase~\cite{bollacker2008freebase}, YAGO2~\cite{hoffart2011yago2},
etc., greatly increase the importance and the commercial value of a QA
system.  Most of such knowledge bases adopt RDF as data
format, and they contain millions or billions of SPO triples
($S$, $P$, and $O$ denote subject, predicate, and object
respectively).

\begin{figure}[h]
\centering
\includegraphics[scale=0.45]{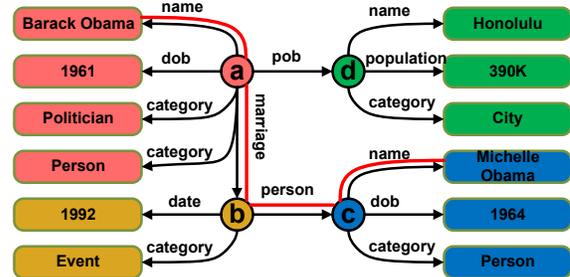}
\caption{A toy RDF knowledge base (here, ``dob'' and ``pob'' stand for ``date of birth'' and
  ``place of birth'' respectively).  Note that the ``spouse of'' intent
  is represented by multiple edges: name - marriage - person -
  name.}
\label{fig:barackobama} 
\vspace{-0.5cm}
\end{figure}


\subsection{Challenges}
\label{sec:intro:challenges}


Given a question against a knowledge base, we face two
challenges: in which representation we understand the questions (representation designment), and how to map the representations to structured queries against the knowledge base (semantic matching)?

\begin{itemize}[leftmargin=0.4cm]
\item {\bf Representation Designment:} Questions describe thousands of intents, and one intent has thousands of question templates. For example, both
  \textcircled{a} and \textcircled{b} in Table~\ref{tab:question} ask
  about {population} of {\it Honolulu}, although they are expressed in quite different ways.
  The QA system needs different representations for different questions. Such representations must be able to (1) identify questions with the same semantics; (2) distinguish different question intents. In the QA corpora we use, we find 27M question templates over 2782 question intents. So it's a big challenge to design representations to handle this.

\item {\bf Semantic Matching:} After figuring out the representation of
  a question, we need to map the representation to a structured query. For BFQ, the structured query mainly depends on the predicate in the knowledge base. Due to the gap between predicates and question representations, it is non-trivial to find such mapping. For example, in
  Table~\ref{tab:question}, we need to know $\textcircled{a}$ has the same semantics with predicate $population$. Moreover, in RDF graph, many binary relations do not correspond to a single edge but a complex structure: in
  Figure~\ref{fig:barackobama}, ``spouse of'' is expressed by a path $marriage \rightarrow$ $person \rightarrow name$. For the knowledge
  base we use, over 98\% intents we found correspond to complex structures.
\end{itemize}
\vspace{-0.3cm}

\begin{table}[!htb]
\setlength{\tabcolsep}{2.5pt}
\scriptsize
\begin{center}
\caption{\bf \small Questions in Natural Language and Related Predicates in a Knowledge Base}
\begin{tabular}{  l | p{2.78cm} }
\hline
Question in Natural language & Predicate in KB \\ \hline
  \hline
  ${\textcircled{a}}$ How many people are there in Honolulu? & population \\ \hline
  ${\textcircled{b}}$ What is the population of Honolulu? & population \\ \hline
  ${\textcircled{c}}$ What is the total number of people in Honolulu? & population    \\ \hline
  ${\textcircled{d}}$ When was Barack Obama born? & dob   \\ \hline
  ${\textcircled{e}}$ Who is the wife of Barack Obama? & marriage$\rightarrow$person$\rightarrow$name \\ \hline
  ${\textcircled{f}}$ When was Barack Obama's wife born? & marriage$\rightarrow$person$\rightarrow$name \\
  & dob \\ \hline
\end{tabular}
\label{tab:question}
\end{center}
\vspace{-0.6cm}
\end{table}

Thus, the key problem is to build a mapping between natural language
questions and knowledge base predicates through proper question representations.

\subsection{Previous Works}
According to how previous knowledge
based QA systems represent questions, we roughly classify them into three categories: rule based,
keyword based, and synonym based.

\begin{enumerate}[leftmargin=0.4cm]
\setlength\itemsep{0em}
\item {\bf Rule based}~\cite{ou2008automatic}. Rule based approaches
  map questions to predicates by using manually constructed rules. This
  leads to high precision but low recall (low coverage of the variety
  of questions), since manually creating rules for a
  large number of questions is infeasible.
\item {\bf Keyword based}~\cite{unger2011pythia}. Keyword based methods
  use keywords in the question and map them to predicates
  by keyword matching. They may answer simple questions such as
  $\textcircled{b}$ in Table~\ref{tab:question} by identifying
  {population} in the question and mapping it to predicate
  {population} in the knowledge base.  But in general, using keywords
  can hardly find such mappings, since one predicate representation in the knowledge
  base cannot match diverse representations in natural language. For example, we cannot find
  {\it population} from $\textcircled{a}$ or $\textcircled{c}$.
\item {\bf Synonym based}~\cite{unger2012template,yahya2012natural,zou2014natural,zheng2015build}.
  Synonym based methods extend keyword based methods by taking
  synonyms of the predicates into consideration. They first generate
  synonyms for each predicate, and then find mappings between
  questions and these synonyms. DEANNA~\cite{yahya2012natural} is a typical synonym based QA system. The main idea is reducing QA into the evaluation of semantic similarity between predicate and candidate synonyms (words/phrases in the question). It uses Wikipedia to compute the semantic similarity. For example, question $\textcircled{c}$ in Table~{1} can be answered by knowing that {\tt number of people} in the question is a synonym of predicate $population$. Obviously, their semantic similarity can be evaluated by Wikipedia. gAnswer~\cite{zou2014natural,zheng2015build} further improved the precision by learning synonyms for more complex sub-structures. However, all these approaches cannot answer ${\textcircled{a}}$ in Table~{1}, as none of {\tt how many}, {\tt people}, {\tt are there} has obvious relation with $population$. {\tt How many people} is ambiguous in different context. In {\tt how many people live in Honolulu?}, it refers to $population$. In {\tt how many people visit New York each year?}, it refers to {\it number of passengers}.
\vspace{-0.2cm}
\end{enumerate}

In general, these works cannot solve the above challenges. For rule based approaches, it takes unaffordable human labeling effort. For keyword based or synonym based approaches, one word or one phrase cannot represent the question's semantic intent completely. We need to understand the question as a whole.
And it's even tremendously more difficult for previous approaches if the question is a complex question or maps to a complex structure in a knowledge base (e.g. $\textcircled{e}$ or $\textcircled{f}$).

\subsection{Overview of Our Approach}

\begin{figure}[h]
\centering
\includegraphics[scale=0.45]{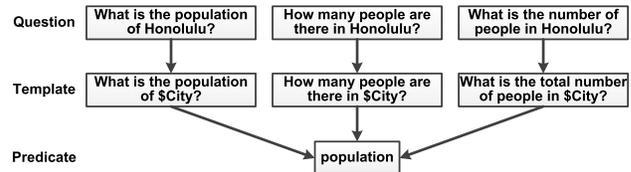}
\caption{\bf Our Approach}
\label{fig:template} 
\vspace{-0.4cm}
\end{figure}

To answer a question, we must first represent the
question. By representing a question, we mean transforming the
question from natural language to an internal representation
that captures the semantics and intent of the question. Then, for each
internal representation, we learn how to map it to an RDF query
against a knowledge base. Thus, the core of our work is the internal representation which we denote as {\it
  templates}.

{\bf Representing questions by templates}
The failure of synonym based approach in ${\textcircled{a}}$ inspires us to understand a question by {\bf templates}.
As an example, {\tt How many people are there in $\$city$?} is the template for ${\textcircled{a}}$. No matter $\$city$ refers to
Honolulu or other cities, the template always asks about population of the question.

Then, the task of representing a question is to map the question to
an existing template. To do this, we replace the entity in the question
by its concepts.  For instance, Honolulu will be replaced by
$\$city$ as shown in Figure~\ref{fig:template}. This process is not
trivial, and it is achieved through a mechanism known as
conceptualization~\cite{song2011short,kim2013context}, which automatically performs
disambiguation on the input (so that the term apple in {\tt what is the
headquarter of apple} will be conceptualized to $\$company$ instead of
$\$fruit$). The conceptualization mechanism itself is based on a large
semantic network (Probase~\cite{wu2012probase}) that consists of
millions of concepts, so that we have enough granularity to
represent all kinds of questions.

The template idea also works for complex questions. Using templates, we simply
decompose the complex question into a series of question, each of which
corresponds to one predicate. Consider question $\textcircled{f}$ in
Table~\ref{tab:question}. We decompose $\textcircled{f}$ into {\tt Barack Obama's wife} and {\tt when was Michelle Obama
born?}, which correspond to $marriage
\rightarrow person \rightarrow name$ and $dob$ respectively. Since the first
question is nested within the second one, we know $dob$ modifies
$marriage\rightarrow person\rightarrow name$, and $marriage\rightarrow
person\rightarrow name$ modifies {Barack Obama}.

{\bf Mapping templates to predicates}
We learn templates and their mappings to knowledge base predicates from
Yahoo! Answers. This problem is quite similar to the semantic parsing~\cite{2013-acl-textual-schema-matching,cai2013semantic}.
Most semantic parsing approaches are synonym based. To model the correlation between phrases and predicates, SEMPRE~\cite{berant2013semantic} uses a bipartite graph, and SPF~\cite{kwiatkowski2013scaling} uses a probabilistic combinatory categorial grammar (CCG)~\cite{clark2007wide}. They still have the drawbacks of synonym based approaches.
{The mapping from templates to predicates is $n:1$,
that is, each predicate in the knowledge base corresponds to multiple
templates.}  For our work, we learned a total of $27,126,355$ different
templates for 2782 predicates. The large amount guarantees the wide coverage of template-based
QA.

The procedure of learning the predicate of a template is as follows.
First, for each QA pair in Yahoo! Answer, we extract the entity in
question and the corresponding value. Then, we find the
predicate from the knowledge base by looking up the {\it direct}
predicate connecting the entity and the value. Our basic idea is, if most instances
of a template share the same predicate, we map the template
to this predicate. For example, suppose questions derived by template
{\tt how many people are there in $\$city$?} always map to the predicate
$population$, no matter what specific $\$city$ it is. We can conclude
that for certain probability the template maps to $population$.
Learning templates that map to a complex knowledge base structure
employs a similar process. The only difference is that we find
``expanded predicates'' that correspond to a {\it path} consisting of
multiple edges which lead from an entity to a certain value (e.g.,
$marriage \rightarrow person \rightarrow name$).

\subsection{Paper Organization}
The rest of the paper is organized as follows. In
Sec~\ref{sec:overview}, we give an overview of KBQA.
The major contribution of this paper is {\bf learning templates} from QA corpora. All technique parts are closely related to it. Sec 3 shows the online question answering with templates. Sec 4 elaborates the predicates inference for templates, which is the key step to use templates. Sec 5 extends our solution to answer a complex question. Sec 6 extends the ability of templates to infer complex predicates.
We present experimental
studies in Sec~\ref{sec:exp}, discuss more related works in
Sec~\ref{sec:related}, and conclude in
Sec~\ref{sec:conslusion}.

\nop{
Sec~\ref{sec:kbqa} illustrates how KBQA works online. In Sec~\ref{sec:pi}, we elaborate the core approach of KBQA, i.e., the inference from template to predicate. In Sec~\ref{sec:complex}, we decompose complex
questions into a series of BFQs to enable the question of those questions.
Sec~\ref{sec:expansion} describes how we handle complex
predicates. We conduct a comprehensive experimental
study in Sec~\ref{sec:exp}, discuss more related works in
Sec~\ref{sec:related}, and conclude in
Sec~\ref{sec:conslusion}.
}

\vspace{-0.2cm}
\section{System Overview}
\label{sec:overview}

In this section, we introduce some background knowledge and give an
overview of KBQA.  In Table~\ref{tab:notation}, we list the
notations used in this paper.

\begin{table}[!htb]
\small
\vspace{-0.2cm}
\begin{center}
\caption{\bf Notations}
\begin{tabular}{  p{0.9cm} | l | p{0.9cm} | l  }
  \hline
  Notation & Description & Notation & Description \\ \hline
  \hline
  $q$ & question & $s$ & subject\\ \hline
  $a$ & answer & $p$ & predicate \\ \hline
  $\mathcal{QA}$ & QA corpus & $o$ & object \\ \hline
  $e$ & entity & $\mathcal{K}$ & knowledge base \\ \hline
  $v$ & value & $c$ & category \\ \hline
  $t$ & template & $p^+$ & expanded predicate \\ \hline
  $V(e,p)$ & $\{v|(e,p,v) \in \mathcal{K} \}$ & $s_2 \subset s_1$ & $s_2$ is a substring of $s_1$ \\ \hline
  $t(q,e,c)$ & template of $q$ by & $\theta^{(s)}$ & estimation of $\theta$  \\
  & conceptualizing $e$ to $c$& & at iteration $s$  \\ \hline
\end{tabular}
\label{tab:notation}
\end{center}
\vspace{-0.5cm}
\end{table}

{\bf Binary factoid QA}
We focus on binary factoid questions (BFQs), that is, questions asking about
a specific property of an entity. For example, all questions except ${\textcircled{f}}$ in Table~\ref{tab:question} are BFQs.

{\bf RDF knowledge base}
Given a question, we find its answer in an RDF knowledge base. An RDF
knowledge base $\mathcal{K}$ is a set of triples in the form of $(s,
p, o)$, where $s$, $p$, and $o$ denote {subject},
{predicate}, and {object} respectively. Figure~\ref{fig:barackobama} shows a toy
RDF knowledge base via an edge-labeled directed graph. Each $(s, p, o)$ is represented by a directed \reviseyq{edge} from $s$ to $o$
labeled with predicate $p$. For example, the edge from $a$ to $1961$ with label $dob$
represents an RDF triple $(a,dob,1961)$, which represents the knowledge of Barack Obama's
birthday.

\begin{center}
\small
\captionof{table}{\bf \small Sample QA Pairs from a QA Corpus}
\begin{tabular}{ c | p{2.9cm} | p{2.9cm} }
\hline
Id & Question & Answer \\ \hline
\hline
$(q_1,a_1)$ & When was Barack Obama born? & The politician was born in 1961.\\ \hline
$(q_2,a_2)$ & When was Barack Obama born? & \reviseyq{He was born in 1961.} \\ \hline
$(q_3,a_3)$ & How many people are there in Honolulu?& It's 390K.\\ \hline
\end{tabular}
\label{tab:qa}
\end{center}

{\bf QA corpora}
We learn question templates from Yahoo! Answer, which consists of 41 million QA pairs.
The QA corpora is denoted by $\mathcal{QA}=\{(q_1,a_1),(q_2,a_2),...,(q_n,a_n)\}$, where $q_i$ is a
question and $a_i$ is the reply to $q_i$. Each reply $a_i$ consists of
several sentences, and the exact factoid answer is
contained in the reply. Table~\ref{tab:qa} shows a sample from a QA corpus.


\paragraph*{{\bf Templates}}
We derive a template $t$ from a question $q$ by replacing each entity
$e$ with one of $e$'s categories $c$. We denote this template as
$t=t(q, e,c)$. 
A question may contain multiple entities, and an entity may
belong to multiple categories. We obtain concept distribution of $e$ through context-aware
conceptualization~\cite{wu2012probase}. For example, question $q_1$ in Table~\ref{tab:qa} contains entity ${\bf a}$ in
Figure~\ref{fig:barackobama}. Since ${\bf a}$ belongs to two
categories: \$Person, \$Politician, we can derive two templates from the question: {\tt When was \$Person born?} and {\tt When was \$Politician
born?}.
\vspace{-0.3cm}

\begin{figure}[h]
\centering \includegraphics[scale=0.4]{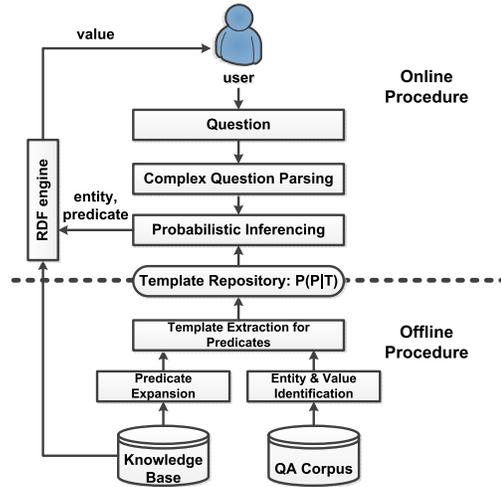}
\caption{\bf System Overview}
\label{fig:framework} 
\vspace{-0.5cm}
\end{figure}

\paragraph*{{\bf System Architecture}}
Figure~\ref{fig:framework} shows the pipeline of our QA system, which consists of two major procedures:
\begin{itemize}[leftmargin=0.4cm]
\setlength\itemsep{0em}
\item \textbf{Online procedure:} When a question comes in, we first
  parse and decompose it into a series of binary factoid
  questions. The decomposition process is described in
  Sec~\ref{sec:complex}. For each binary factoid question, we use
  a probabilistic inference approach to find its value, shown in Sec~\ref{sec:kbqa}. The
  inference is based on the predicate
  distribution of given templates, i.e. $P(p|t)$. Such distribution is learned offline.

\item \textbf{Offline procedure:} The goal of offline procedure is to
  learn the mapping from templates to predicates. This is represented by $P(p|t)$, which is estimated in Sec~\ref{sec:pi}.
 And we expand predicates in the knowledge base in Sec~\ref{sec:expansion}, so that we can learn more complex predicate forms (e.g., $marriage
  \rightarrow person \rightarrow name$ in
  Figure~\ref{fig:barackobama}).
%
\vspace{-0.3cm}
\end{itemize}

\section{Our Approach: KBQA}
\label{sec:kbqa}
In this section, we first formalize our problem in a probabilistic framework in Sec~\ref{sec:probabilisticmodel}.
We present the details for most probability estimations in Sec~\ref{sec:otherdistributions},  leaving only the estimation of $P(p|t)$ in Sec~\ref{sec:pi}.
We elaborate the online procedure in Sec~\ref{sec:online}.


\subsection{Problem Model}
\label{sec:probabilisticmodel}
KBQA learns question answering by using a QA corpus and a
knowledge base. Due to issues such as {\it uncertainty} (e.g. some questions' intents are vague), {\it incompleteness} (e.g. the knowledge base is almost
always incomplete), and {\it noise} (e.g.  answers in the QA corpus
may be wrong), we create a probabilistic model for QA over a knowledge base below. We highlight the uncertainty from the question's intent to the knowledge base's predicates~\cite{kwiatkowski2013scaling}. For example, the question ``where was Barack Obama from'' is related to at least two predicates in Freebase: ``place of birth'' and ``place lived location''. In DBpedia, {\tt who founded \$organization?} relates to predicates $founder$ and $father$.


\begin{pdef}
\label{pdef:pvq}
Given a question $q$, our goal is
to find an answer $v$ with maximal
probability ($v$ is a simple value):
\begin{equation}
\small
\argmax_{v}P(V=v|Q=q)
\end{equation}
\vspace{-0.7cm}
\end{pdef}



To illustrate how a value is found for a given question, we proposed a generative model. Starting from the user question $q$, we first generate/identify its entity $e$ according to the distribution $P(e|q)$. After knowing the question and the entity, we generate the template $t$ according to the distribution $P(t|q,e)$. The predicate $p$ only depends on $t$, which enables us to infer the predicate $p$ by $P(p|t)$. Finally, given the entity $e$ and the predicate $p$, we generate the answer value $v$ by $P(v|e,p)$. $v$ can be directly returned or embedded in a natural language sentence as the answer $a$. We illustrate the generation procedure in Example~\ref{exam:generative}, and shows the dependency of these random variables in Figure~\ref{fig:onlineplate}.
Based on the generative model, we compute $P(q,e,t,p,v)$ in Eq~(\ref{eqn:pqetpv}). Now Problem~\ref{pdef:pvq} is reduced to Eq~(\ref{eqn:vreduce}).
\begin{equation}
\label{eqn:pqetpv}
P(q,e,t,p,v)=P(q)P(e|q)P(t|e,q)P(p|t)p(v|e,p)
\end{equation}
\begin{equation}
\label{eqn:vreduce}
\argmax_{v} \sum_{e,t,p} P(v|q,e,t,p)
\end{equation}
\begin{example}
\begin{figure}[h]
\vspace{-0.5cm}
\centering
\includegraphics[scale=0.4]{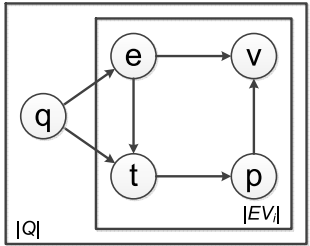}
\caption{\bf Probabilistic Graph}
\label{fig:onlineplate} 
\vspace{-0.3cm}
\end{figure}
\label{exam:generative}
Consider the generative process of $(q_3,a_3)$ in Table~\ref{tab:qa}. Since the only entity in $q_3$ is ``Honolulu'', we generate the entity node $d$ (in Figure~\ref{fig:barackobama}) by $P(e=d|q=q_3)=1$. By conceptualizing ``Honolulu'' to a city, we generate the template {\tt How many people are there in \$city?}. Note that the corresponding predicate of the template is always ``population'', no matter which specific city it is. So we generate predicate ``population'' by distribution $P(p|t)$. After generating entity ``Honolulu'' and predicate ``population'', the value ``390k'' can be easily found from the knowledge base in Figure~\ref{fig:barackobama}. Finally we use a natural language sentence $a_3$ as the answer.  
\end{example}

{\bf Outline of the following subsections}
Given the above objective function, our problem is reduced to the estimation of each probability term in Eq~\eqref{eqn:pqetpv}.
The term $P(p|t)$ is estimated in the offline procedure in Sec~\ref{sec:pi}.
All other probability terms can be directly computed by the off-the-shelf solutions (such as NER, conceptualization). We elaborate the calculation of these probabilities in Sec~\ref{sec:otherdistributions}. And we elaborate the online procedure in Sec~\ref{sec:online}.


\subsection{Probability Computation}
\label{sec:otherdistributions}
In this subsection, we compute each probability term in Eq~\eqref{eqn:pqetpv} except $P(p|t)$.

{\bf Entity distribution $P(e|q)$} The distribution represents the entity identification from the question. We identify entities that meet both conditions: (a) it is an entity in the question; (b) it is in the knowledge base. We use Stanford Named Entity Recognizer~\cite{finkel2005incorporating} for (a). And we then check if it is an entity's name in the knowledge base for (b). If there
are multiple candidate entities, we simply give them
uniform probability.

We optimize the computation of $P(e|q)$ in the offline procedure by $q$'s answer. As illustrated in Sec~\ref{sec:observationformalization}, we already extracted a set of entity-value pairs $EV_i$ for question $q_i$ and answer $a_i$, where the values are from the answer. We assume the entities in $EV_i$ have equal probability to be generated. So we obtain:
\begin{equation}
\label{eqn:peq}
P(e|q_i)=\frac{ [\exists v, (e,v) \in EV_i]  }{|\{e'|\exists v, (e',v) \in EV_i \}|}
\end{equation}
,where $[.]$ is the Iverson bracket. As shown in Sec~\ref{sec:exp:components}, this approach is more accurate than directly using the NER approach.

%

{\bf Template distribution $P(t|q,e)$}
A template is in the form of {\tt When was \$person
born?}. In other words, it is a question with the mention of an
entity (e.g., ``Barack Obama'') replaced by the category of the entity
(e.g., \$person).

Let $t = t(q, e, c)$ indicate that template $t$ is obtained by
replacing entity $e$ in $q$ by $e$'s category $c$. Thus, we have
\begin{equation}
\small
\label{eqn:ptqe}
P(t|q, e) = P(c|q, e)
\end{equation}
, where $P(c|q, e)$ is the category distribution of $e$ in context $q$. In our work, we directly apply the
conceptualization method in~\cite{song2011short} to compute $P(c|q,e)$.




{\bf Value (answer) distribution $P(v|e,p)$}
For an entity $e$ and a predicate $p$ of $e$, it is
easy to find the predicate value $v$ by looking up the knowledge base. For
example, in Figure~\ref{fig:barackobama}, let entity $e$ = Barack Obama, and predicate $p$ = dob. We easily get Obama's birthday, 1961, from the knowledge base. In this
case, we have $P(1961|\mbox{Barack Obama}, dob) = 1$, since Barack Obama
only has one birthday. Some predicates may
have multiple values (e.g., the children of Barack Obama). In this
case, we assume uniform probability for all possible values. More formalized, we compute $P(v|e,p)$ by
\begin{equation}
\label{eqn:pvep}
P(v|e,p)=\frac{[(e,p,v) \in \mathcal{K} ]}{|\{(e,p,v')|(e,p,v') \in \mathcal{K}\}|}
\end{equation}

\subsection{Online Procedure}
\label{sec:online}
In the online procedure, we are given a user question $q_0$. We can compute $p(v|q_0)$ by Eq~\eqref{eqn:pvq}. And we return $\argmax_{v}P(v|q_0)$ as the answer.
\begin{flalign}
\begin{split}
\label{eqn:pvq}
P(v|q_0) =  \sum_{e,p,t} P(q_0)P(v|e,p)P(p|t)P(t|e,q_0)P(e|q_0)
\end{split}
\end{flalign}
, where $P(p|t)$ is derived from offline learning in Sec~\ref{sec:pi}, and other probability terms are computed in Sec~\ref{sec:otherdistributions}.


{\bf Complexity of Online Procedure:} In the online procedure, we enumerate $q_0$'s entities, templates, predicates, and values in order. We treat the number of entities per question, the number of concepts per entity, and the number of values per entity-predicate pair as constants. So the complexity of the online procedure is $O(|P|)$, which is caused by the enumeration on predicate. Here $|P|$ is the number of distinct predicates in the knowledge base.

\section{Predicate Inference}
\label{sec:pi}
In this section, we present how we infer predicates from templates, i.e., the estimation of $P(p|t)$.
We treat the distribution $P(P|T)$ as parameters and then use the maximum likelihood (ML) estimator to estimate $P(P|T)$. To do this, we first formulate the likelihood of the observed data (i.e., QA pairs in the corpora) in Sec~\ref{sec:observationformalization}. Then we present the parameter estimation and its algorithmic implementation in Sec~\ref{sec:templatelearning} and Sec~\ref{sec:implementation}, respectively.

\subsection{Likelihood Formulation}
\label{sec:observationformalization}
Rather than directly formulating the likelihood to observe the QA corpora ($\mathcal{QA}$ ), we first formulate a simpler case, the likelihood of a set of question-entity-value triples extracted from the QA pairs. Then we build the relationship between the two likelihoods. The indirect formulation is well motivated. An answer in $\mathcal{QA}$ is usually a complicated natural language sentence containing the exact value and many other tokens.
Most of these tokens are meaningless in indicating the predicate and bring noise to the observation.
On the other hand directly modeling the complete answer in a generative model is difficult, while modeling the value in a generative model is much easier.

Next, we first extract entity-value pairs from the given QA pair in Sec~\ref{sec:entityvalueextraction}, which allows to formalize the
likelihood of question-entity-value triples ($X$). We then establish the relationship between the likelihood of the QA corpora and the likelihood of $X$, in Eq~\eqref{eqn:likelihoodequal}, Sec~\ref{sec:likelihoodfunction}.





\subsubsection{Entity-Value Extraction}
\label{sec:entityvalueextraction}
Our principle to extract candidate values from the answer is that {\it a valid entity\&value pair usually has some corresponding relationships in knowledge base}.
Following this principle, we identify the candidate entity\&value pairs from $(q_i,a_i)$:
\begin{equation}
\label{eqn:vqa}
EV_i=\{(e,v)|e \subset q_i, v \subset a_i, \exists p, (e,p,v) \in \mathcal{K}\}
\end{equation}, where $\subset$  means ``is substring of''. We illustrate this in Example~\ref{example:vqa}.
\begin{example}
\label{example:vqa}
Consider $(q_1,a_1)$ in Table~\ref{tab:qa}. Many tokens (e.g. {\tt The, was, in}) in the answer are useless. We extract the valid value {\tt 1961}, by noticing that the entity {\tt Barack Obama} in $q_1$ and {\tt 1961} are connected by predicate ``pob'' in Figure~\ref{fig:barackobama}. Note that we also extract the noise value {\tt politician} in this step. We will show how to filter it in the refinement step below.
\end{example}
{\bf Refinement of $EV_i$} We need to filter the noisy pairs in $EV(q,a)$, e.g. $(Barack\ Obama, politician)$ in Example~\ref{example:vqa}. The intuition is: the correct value and the question should {\it have the same category.} Here the category of a question means the expected answer type of the question. This has been studied as the question classification problem~\cite{metzler2005analysis}. We use the UIUC taxonomy~\cite{li2002learning}. For question categorization, we use the approach proposed in~\cite{metzler2005analysis}. For value's categorization, we refer to the category of its predicate. The predicates' categories are manually labeled. This is feasible since there are only a few thousand predicates.

\subsubsection{Likelihood Function}
\label{sec:likelihoodfunction}
After the entity\&value extraction, each QA pair $(q_i, a_i)$ is transferred into a question and a set of entity-value pairs, i.e., $EV_i$. By assuming the independence of these entity-value pairs, the probability of observing such a QA pair is shown in Eq~\eqref{eqn:pqiai}. Thus, we compute the likelihood of the entire QA corpora in Eq~\eqref{eqn:lqa}.
\begin{equation}
\label{eqn:pqiai}
\small
P(q_i,a_i)=P(q_i, EV_i)= P(q_i) \prod_{(e,v) \in EV_i} P(e,v|q_i)
\end{equation}
\begin{equation}
\label{eqn:lqa}
L_{\mathcal{QA}}=\prod_{i=1}^n [ P(q_i) \prod_{(e,v) \in EV_i}P(e,v|q_i) ]
\end{equation}
By assuming each question has an equal probability to be generated, i.e. $P(q_i)=\alpha$, we have:
\begin{flalign}
\small
\label{eqn:likelihood}
\begin{split}
& L_{\mathcal{QA}}=\prod_{i=1}^n [ P(q_i)^{1-|EV_i|} \prod_{(e,v) \in EV_i} P(e,v|q_i) P(q_i) ] \\
& =\beta \prod_{i=1}^n [  \prod_{(e,v) \in EV_i} P(e,v,q_i) ]
\end{split}
\end{flalign}
, where $\beta=\alpha^{n - \sum_{i=1}^n |EV_i| }$ can be considered as a constant. Eq~\eqref{eqn:likelihood} implies that $L_{\mathcal{QA}}$ is proportional to the likelihood of these
question-entity-value triples. Let $X$ be the set of such triples that are extracted from QA corpora:
\begin{equation}
\small
X= \{(q_i,e,v)|(q_i,a_i)\in \mathcal{QA}, (e, v)\in EV_i\}
\end{equation}
We denote the $i$-the term in $X$ as $x_i=(q_i,e_i,v_i)$. So $X=\{x_1,...,x_m\}$. Thus we establish the linear relationship between the likelihood of $\mathcal{QA}$ and the likelihood of $X$.
\begin{flalign}
\small
\label{eqn:likelihoodequal}
\begin{split}
L_{\mathcal{QA}}= \beta L_{X}=\beta \prod_{i=1}^m P(x_i)=\beta \prod_{i=1}^m P(q_i,e_i,v_i)
\end{split}
\end{flalign}

Now, maximizing the likelihood of $\mathcal{QA}$ is equivalent to maximize the likelihood of $X$.
Using the generative model in Eq~\eqref{eqn:pqetpv}, we calculate $P(q_i,e_i,v_i)$ by marginalizing the joint probability  $P(q,e,t,p,v)$
over all templates $t$ and all predicates $p$.
The likelihood is shown in Eq~\eqref{eqn:likelihoodx}. We illustrate the entire process in Figure~\ref{fig:onlineplate}.
\begin{flalign}
\small
\begin{split}
L_{X}=\prod_{i=1}^m \sum_{p\in P,t \in T} P(q_i)P(e_i|q_i)P(t|e_i,q_i)P(p|t)p(v_i|e_i,p)
\end{split}
\label{eqn:likelihoodx}
\end{flalign}
\vspace{-0.3cm}




\subsection{Parameter Estimation}
\label{sec:templatelearning}

{\bf Goal:} In this subsection, we estimate $P(p|t)$ by maximizing Eq~\eqref{eqn:likelihoodx}. 
We denote the distribution $P(P|T)$ as parameter $\theta$ and its corresponding log-likelihood as $L(\theta)$. And we denote the probability $P(p|t)$ as $\theta_{pt}$. So we estimate $\theta$ by:
\begin{equation}
\small
\hat{\theta}=\argmax_{\theta} L(\theta)
\end{equation}
, where
\begin{flalign}
\small
\begin{split}
& L(\theta)=\sum_{i=1}^m \log P(x_i) = \sum_{i=1}^m \log P(q_i,e_i,v_i) \\
& =\sum_{i=1}^m \log [\sum_{p\in P,t \in T} P(q_i)P(e_i|q_i)P(t|e_i,q_i)\theta_{pt} P(v_i|e_i,p)]
\end{split}
\end{flalign}

{\bf Intuition of EM Estimation:} We notice that some random variables (e.g. predicate and template) are latent in the proposed probabilistic model, which motivates us to use the Expectation-Maximization (EM) algorithm to estimate the parameters. The EM algorithm is a classical approach to find the maximum likelihood estimates of parameters in a statistical model with unobserved variables.
The ultimate objective is to maximize the \emph{likelihood of complete data} $L(\theta)$. However, it involves a logarithm of a sum and is computationally hard.
Hence, we instead resort to maximizing one of its lower bound~[7], i.e., the {\bf Q-function} $\mathcal{Q}(\theta;\theta^{(s)})$.
To define the Q-function, we leverage the \emph{likelihood of complete data} $L_c(\theta)$. The EM algorithm maximizes $L(\theta)$ by maximizing the lower bound $\mathcal{Q}(\theta;\theta^{(s)})$ iteratively. In the $s$-th iteration, {\bf E-step} computes $\mathcal{Q}(\theta;\theta^{(s)})$ for given parameters $\theta^{(s)}$, and {\bf M-step} estimates the parameters $\theta^{(s+1)}$ (parameters in the next iteration) that maximizes the lower bound.

{\bf Likelihood of Complete Data:} Directly maximizing $L(\theta)$ is computationally hard, since the function involves a logarithm of a sum. Intuitively, if we know the complete data of each observed triple, i.e., which template and predicate it is generated with, the estimation becomes much easier. We thus introduce a hidden variable $z_i$ for each observed triple $x_i$. The value of $z_i$ is a pair of predicate and template, i.e. $z_i=(p,t)$, indicating $x_i$ is generated with predicate $p$ and template $t$. Note that we consider predicate and template together since they are not independent in the generation. Hence, $P(z_i=(p,t))$ is the probability that $x_i$ is generated with predicate $p$ and template $t$.

We denote $Z=\{z_1,...,z_m\}$. $Z$ and $X$ together form the complete data. The log-likelihood of observing the complete data is:
\begin{equation}
\small
L_c(\theta)=\log P(X,Z|\theta) = \sum_{i=1}^m \log P(x_i,z_i|\theta)
\end{equation}
, where
\begin{flalign}
\small
\begin{split}
& P(x_i,z_i=(p,t)|\theta)=P(q_i,e_i,v_i,p,t|\theta) \\
&=P(q_i)P(e_i|q_i)P(t|e_i,q_i)\theta_{pt}P(v_i|e_i,p) =f(x_i,z_i)\theta_{pt}
\end{split}
\label{eqn:p2f}
\end{flalign}
\begin{equation}
\small
\label{eqn:f}
f(x=(q,e,v),z=(p,t))=P(q)P(e|q)P(t|e,q)P(v|e,p)
\end{equation}
As discussed in Sec~\ref{sec:otherdistributions}, $f()$ can be computed independently before the estimation of $P(p|t)$. So we treat it as a known factor.

{\bf Q-function:}
Instead of optimizing $L(\theta)$ directly, we define the ``Q-function'' in Eq~\eqref{eqn:qfunc}, which is the expectation of the complete observation's likelihood. Here $\theta^{(s)}$ is the estimation of $\theta$ at iteration $s$. According to Theorem~\ref{proposition:1}, when treating $h(\theta^{(s)})$ as a constant, $\mathcal{Q}(\theta;\theta^{(s)})$ provides a lower bound for $L(\theta)$. Thus we try to improve $\mathcal{Q}(\theta;\theta^{(s)})$ rather than directly improve $L(\theta)$.
\begin{flalign}
\small
\label{eqn:qfunc}
\begin{split}
&\mathcal{Q}(\theta;\theta^{(s)})= E_{P(Z|X,\theta^{(s)})}[ L_c(\theta) ] \\
& = \sum_{i=1}^m \sum_{p \in P, t \in T} P(z_i=(p,t)|X,\theta^{(s)}) \log P(x_i, z_i=(p,t)|\theta)
\end{split}
\end{flalign}
\begin{theorem}[ Lower bound~\cite{dempster1977maximum}]
\label{proposition:1}
$L(\theta) \ge \mathcal{Q}(\theta;\theta^{(s)}) + h(\theta^{(s)})$
, where $h(\theta^{(s)})$ only relies on $\theta^{(s)}$ and can be treated as constant for $L(\theta)$.\end{theorem}

In {\bf E-step}, we compute $\mathcal{Q}(\theta;\theta^{(s)})$. For each $P(z_i|X,\theta^{s})$ in Eq~\eqref{eqn:qfunc}, we have:
\begin{flalign}
\small
\label{eqn:estep}
\begin{split}
P(z_i=(p,t)|X,\theta^{(s)}) = f(x_i,z_i) \theta_{pt}^{(s)}
\end{split}
\end{flalign}

In {\bf M-step}, we maximize the Q-function. By using Lagrange multiplier, we obtain $\theta_{pt}^{(s+1)}$ in Eq~\eqref{eqn:mstep}. 
\begin{equation}
\small
\label{eqn:mstep}
\theta_{pt}^{(s+1)}= \frac{ \sum_{i=1}^m  P(z_i=(p,t)|X,\theta^{(s)})}{\sum_{p' \in P} \sum_{i=1}^m  P(z_i=(p',t)|X,\theta^{(s)})}
\end{equation}

\nop{
Since $P(P|t)$ is a probabilistic distribution, we have an obvious constraint: for each $t$, $\sum_{p\in P} \theta_{pt}=1$. Using the method of Lagrange multiplier, we introduce a new function $g(\theta)$ in Eq~\eqref{eqn:g}, where each $\mu_t$ is a Lagrange multiplier:
\begin{equation}
\small
\label{eqn:g}
g(\theta)=\mathcal{Q}(\theta;\theta^{(s)})+\sum_t \mu_t (1-\sum_p \theta_{pt})
\end{equation}
According to the method of Lagrange multipliers, the $\theta$ of $g(\theta)$'s stationary point is a maximum parameter for $\mathcal{Q}(\theta;\theta^{(s)})$. So we take its derivative w.r.t. each parameter variable $\theta_{pt}$:
\begin{flalign}
\small
\begin{split}
& \frac{\partial g(\theta)}{\partial \theta_{pt}} = \sum_{i=1}^m \frac{P(z_i=(p,t)|X,\theta^{(s)})}{P(x_i,z_i=(p,t)|\theta)} \frac{ \partial P(x_i,z_i=(p,t)|\theta)  }{\partial \theta_{pt}} -\mu_t \\
& = \sum_{i=1}^m \frac{P(z_i=(p,t)|X,\theta^{(s)})}{ f(x_i,(p,t))\theta_{pt} } f(x_i,(p,t)) -\mu_t \\
& = \sum_{i=1}^m \frac{P(z_i=(p,t)|X,\theta^{(s)})}{ \theta_{pt} } -\mu_t
\end{split}
\end{flalign}

After setting this derivative to zero and solving the equation for $\theta_{pt}$, we obtain $\theta_{pt}^{(s+1)}$ as follows:
\begin{equation}
\label{eqn:mstep}
\theta_{pt}^{(s+1)}= \frac{ \sum_{i=1}^m  P(z_i=(p,t)|X,\theta^{(s)})}{\sum_{p' \in P} \sum_{i=1}^m  P(z_i=(p',t)|X,\theta^{(s)})}
\end{equation}
}

\subsection{Implementation}
\label{sec:implementation}

Now we elaborate the implementation of the EM algorithm (in Algorithm~\ref{algo:em}), which consists of three steps: initialization, E-step, and M-step.

{\bf Initialization}: 
To avoid $P(z_i=(p,t)|X,\theta^{(s)})$ in Eq~\eqref{eqn:estep} being all zero, we require that $\theta^{(0)}$ is uniformly distributed over all pairs of $(x_i, z_i)$ s.t. $f(x_i,z_i)>0$. So we have:
\begin{equation}
\theta_{pt}^{(0)}= \frac{[\exists i, f(x_i,z_i=(p,t)) > 0]}{ |\{p'| \exists i, f(x_i,z_i=(p',t))>0 \}| }
\end{equation}

{\bf E-step}: We enumerate all $z_i$ and compute $P(z_i|X,\theta^{(s)})$ by Eq~\eqref{eqn:estep}. Its complexity is $O(m)$.


{\bf M-step}:
We compute the $\sum_{i=1}^m P(z_i=(p,t)|X,\theta^{(s)})$ for each $\theta_{pt}^{(s+1)}$.
The direct computation costs $O(m|P||T|)$ time since we need to enumerate all possible templates and predicates. Next, we reduce it to $O(m)$ by only enumerating a constant number of templates and predicates for each $i$.

We notice that only $z_i$ with $P(z_i=(p,t)|X,\theta^{(s)})>0$ needs to be considered. Due to Eq~\eqref{eqn:p2f} and Eq~\eqref{eqn:f}, this implies:
\begin{flalign}
\begin{split}
f(x_i,z_i=(p,t)) > 0  \Rightarrow P(t|e_i,q_i)>0, P(v_i|e_i,p)>0
\end{split}
\end{flalign}
With $P(t|e_i,q_i)>0$, we pruned the enumeration of templates.
$P(t|e_i,q_i)>0$ implies that we only enumerate the templates which are derived by conceptualizing $e_i$ in $q_i$. The number of concepts for $e$ is obviously upper bounded and can be considered as a constant. Hence, the total number of templates $t$ enumerated in Line 7 is $O(m)$. With $P(v_i|e_i,p)>0$, we pruned the enumeration of predicates. $P(v_i|e_i,p)>0$ implies that only predicates connecting $e_i$ and $v_i$ in the knowledge base need to be enumerated. The number of such predicates also can be considered as a constant. So the complexity of the M-step is $O(m)$.

\IncMargin{1em}
\begin{algorithm}
\small
 \KwData{$X$;}
 \KwResult{$P(p|t)$;}
    Initialize the iteration counter $s \leftarrow 0$; \\
    Initialize the parameter $\theta^{(0)}$; \\
    \While{$\theta$ not converged}{
        \nonl //E-step \;
        \For {$i=1...m$}{
            Estimate $P(z_i|X,\theta^{(s)})$ by Eq~\eqref{eqn:estep} \;
        }
        \nonl //M-step \;
        \For {$i=1...m$ }{
            \For {all $t \in T$ for $q_i,e_i$ with $P(t|q_i,e_i)>0$ \label{line:em:1}}{
                \For {all $p \in P$ with $P(v_i|e_i,p)>0$}{
                    $\theta_{pt}^{(s+1)}+=P(z_i=(p,t)|X,\theta^{(s)})$ \;
                }
            }
        }
        Normalize $\theta_{pt}^{(s+1)}$ as in Eq~\eqref{eqn:mstep} \;
        $s +=1$ \;
    }
    \Return {$P(p|t)$}
 \caption{\bf \scriptsize EM Algorithm for Predicate Inference}
 \label{algo:em}
\end{algorithm}

{\bf Overall Complexity of EM algorithm:} Suppose we repeat the EM algorithm $k$ times, the overall complexity thus is $O(km)$.


\section{Answering Complex Questions}
\label{sec:complex}

In this section, we elaborate how we answer complex questions.
We first formalize the problem as an optimization problem in Sec 5.1. Then, we elaborate the optimization metric
and our algorithm in Sec 5.2 and Sec 5.3, respectively.


\subsection{Problem Statement}
\label{sec:complex:prob}
We focus on the {\it complex questions} which are composed of a sequence of BFQs. For example, question ${\textcircled{f}}$ in Table~\ref{tab:question} can be decomposed into two BFQs: (1) Barack Obama's wife (Michelle Obama); (2) When was Michelle Obama born? (1964).
Clearly, the answering of the second question relies on the answer of the first question.

A {\bf divide-and-conquer framework} can be naturally leveraged to answer complex questions: (1) we first decompose the question into a sequence of BFQs, (2) then we answer each BFQ sequentially. Since we have shown how to answer BFQ in Sec~\ref{sec:kbqa}, the key issue in this section is the decomposition.

We highlight that in the decomposed question sequence, each question except the first one is a question string with an entity variable. The question sequence can only be materialized after the variable is assigned with a specific entity, which is the answer of the immediately previous question. Continue the example above, the second question {\tt When was Michelle Obama born?} is {\tt When was $\$e$ born?} in the question sequence. $\$e$ here is the variable representing the answer of the first question {\tt Barack Obama's wife}.
Hence, given a complex question $q$, we need to decompose it into a sequence of $k$ questions $\mathcal{A}=(\check{q}_i)_{i=0}^k$ such that:

\begin{itemize}[leftmargin=0.4cm]
\setlength\itemsep{0em}
\item Each $\check{q}_i$ ($i>0$) is a BFQ with entity variable $e_i$, whose value is the answer of $\check{q}_{i-1}$.
\item $\check{q}_0$ is a BFQ that its entity is equal to the entity of $q$.
\end{itemize}
\begin{example}[Question sequence]
\label{exam:ts}
Consider the question $\textcircled{f}$ in Table~\ref{tab:question}. One natural question sequence is $\check{q}_0$= {\tt Barack Obama's wife} and $\check{q}_1=$ {\tt When was $\$e_1$ born?}.
We can also substitute an arbitrary substring to construct the question sequence, such as $\check{q}'_0$ ={\tt was Barack Obama's wife born} and $\check{q}'_1$={\tt When $\$e$?}. However, the later question sequence is invalid since $\check{q}'_0$ is neither an answerable question nor a BFQ.
\end{example}

Given a complex question, we construct a question sequence in a recursive way. We first replace a substring with an entity variable. If the substring is a BFQ that can be directly answered, it is $q_0$. Otherwise, we continue the above procedure on the substring until we meet a BFQ or the substring is a single word. However, as shown in Example~\ref{exam:ts}, many question decompositions are not valid (answerable). Hence, we need to measure how likely a decomposition sequence is answerable.
More formally, let $\mathbb{A}(q)$ be the set of all possible decompositions of $q$.
For a decomposition $\mathcal{A} \in \mathbb{A}(q)$, let $P(\mathcal{A})$ be the probability that $\mathcal{A}$ is a valid (answerable) question sequence. Out problem thus is reduced to
\begin{equation}
\small
\label{eqn:argmaxvqs}
\mathop{\argmax}_{\mathcal{A}\in \mathbb{A}(q)} P(\mathcal{A})
\end{equation}

Next, we elaborate the estimation of $P(\mathcal{A})$ and how we solve the optimization problem efficiently in Sec~\ref{sec:complex:metric} and~\ref{sec:complex:solution}, respectively.



\subsection{Metric}
\label{sec:complex:metric}

The basic intuition is that $\mathcal{A}=(\check{q}_i)_{i=0}^k$ is a valid question sequence if each individual question $\check{q}_i$ is valid. Hence, we first estimate $P(\check{q}_i)$ (the probability that $q_i$ is a valid question), and then aggregate each $P(\check{q}_i)$ to compute $P(\mathcal{A})$.

We use QA corpora to estimate $P(\check{q}_i)$.
$\check{q}$ is a BFQ with entity variable $\$e$. A question $q$ matches $\check{q}$, if we can get $\check{q}$ by replacing a substring of $q$ with \$e. We say \emph{the match is \emph{valid}, if the replaced substring is a mention of the entity in $q$}. For example, {\it When was Michelle Obama born?} matches {\it When was \$e born?} and {\it When was \$e?}. However, only the former one is valid since only {\it Michelle Obama} is an entity.
We denote the number of all questions in the QA corpora that matches $\check{q}$ as $f_o(\check{q})$, and the number of questions that validly matches $\check{q}$ as $f_v(\check{q})$. Both $f_v(\check{q}_i)$ and $f_o(\check{q}_i)$ are counted by the QA corpora. We estimate $P(\check{q}_i)$ by:
\begin{equation}
\small
\label{eqn:pvalidq}
P(\check{q}_i)=\frac{f_v(\check{q}_i)}{f_o(\check{q}_i)}
\end{equation}
The rationality is clear: the more valid match the more likely $\check{q}_i$ is answerable. $f_o(\check{q}_i)$ is used to punish the over-generalized question pattern. We show an example of $P(\check{q}_i)$ below.
\begin{example}
Suppose $\check{q}_1=$ {\tt When was $\$e$ born?}, $\check{q}_2=$ {\tt When $\$e$?}, the QA corpora is shown in Table~\ref{tab:qa}. Clearly, $q_1$ satisfies the patterns of $\check{q}_1$ and $\check{q}_2$. However, only $\check{q}_1$ is a valid pattern for $q_1$ since when matching $q_1$ to $\check{q}_1$, the replaced substring corresponds to a valid entity ``Barack Obama''. Thus we have $f_v(\check{q}_1)=f_o(\check{q}_1)=f_o(\check{q}_2)=2$. However, $f_v(\check{q}_0)=0$. Due to Eq~\eqref{eqn:pvalidq}, $P(\check{q}_1)=1$, $P(\check{q}_2)=0$. 
\end{example}

Given each $P(\check{q}_i)$,  we define $P(\mathcal{A})$. We assume that each $\check{q}_i$ in $\mathcal{A}$ being valid are independent. A question sequence $\mathcal{A}$ is valid if and only if all $\check{q}_i$ in it are valid. So we compute $P(\mathcal{A})$ by:
\begin{equation}
\small
P(\mathcal{A})=\prod_{\check{q} \in \mathcal{A}} P(\check{q})
\end{equation}

\subsection{Algorithm}
\label{sec:complex:solution}

Given $P(\mathcal{A})$, our goal is to find the question sequence maximizing $P(\mathcal{A})$. This is not trivial due to the huge search space. Consider a complex question $q$ of \emph{length} $|q|$, i.e., the number of words in $q$. There are overall $O(|q|^2)$ substrings of $q$. If $q$ finally is decomposed into $k$ sub-questions, the entire search space will be $O(|q|^{2k})$, which is unacceptable. In this paper, we proposed a dynamic programming based solution to solve our optimization problem, with complexity $O(|q|^4)$. Our solution is developed upon the {\it local optimality} property of the optimization problem.
We establish this property in Theorem~\ref{theo:bpp}.

\begin{theorem}[Local Optimality]
\label{theo:bpp}
Given a complex question $q$, let $\mathcal{A}^*(q)=(\check{q}_0^*, ..., \check{q}_k^*)$ be the optimal decomposition of $q$, then $\forall 1 \le i \le k$, $\exists q_i \subset q$, $\mathcal{A}^*(q_i)=(\check{q}_0^*,..,\check{q}_i^*)$ is the optimal decomposition of $q_i$.
\end{theorem}

Theorem~2 suggests a dynamic programming (DP) algorithm.
Consider a substring $q_i$ of $q$, $q_i$ is either (1) a primitive BFQ (non-decomposable) or (2) a string that can be further decomposed. For case (1),  $\mathcal{A}^*(q_i)$ contains a single element, i.e., $q_i$ itself. For case (2), $\mathcal{A}^*(q_i)=\mathcal{A}^*(q_j)\oplus r(q_i, q_j)$, where $q_j \subset q_i$ is the one with the maximal $P(r(q_i,q_j))P(\mathcal{A}^*(q_j))$, $\oplus$ is the operation that appends a question at the end of a question sequence, and $r(q_i, q_j)$ is the question generated by replacing $q_j$ in $q_i$ with a placeholder ``\$e''.
Thus, we derive the dynamic programming equation:
\begin{equation}
\small
P(\mathcal{A}^*(q_i))=\max\{
\delta(q_i),
\max_{q_j \subset q_i} \{ P(r(q_i,q_j))P(\mathcal{A}^*(q_j)) \}  \}
\end{equation}
where $\delta(q_i)$ is the indicator function to determine whether $q_1$ is a primitive BFQ. That is $\delta(q_i)=1$ when $q_i$ is a primitive BFQ, or $\delta(q_i)=0$ otherwise.

Algorithm~\ref{alg:bpp} outlines our dynamic programming algorithm. We enumerate all substrings of $q$ in the outer loop (Line~\ref{line:1}). Within each loop, we first initialize $\mathcal{A}^*(q_i)$ and $P(\mathcal{A}^*(q_i))$ (Line~2-4). In the inner loop, we enumerate all substrings $q_j$ of $q_i$ (Line~\ref{line:2}), and update $\mathcal{A}^*(q_i)$ and $P(\mathcal{A}^*(q_i))$ (Line~\ref{line:4}-\ref{line:bpp:6}).
Note that we enumerate all $q_i$s \emph{in the ascending order of their lengths}, which ensures that $P(\mathcal{A}^*())$ and $\mathcal{A}^*()$ are known for each enumerated $q_j$.

The complexity of Algorithm~\ref{alg:bpp} is $O(|q|^4)$, since both loops enumerates $O(|q|^2)$ substrings. In our QA corpora, over 99\% questions contain less than 23 words ($|q|<23$). So this complexity is acceptable.

\setcounter{algocf}{1}
\IncMargin{1em}
\begin{algorithm}
\small
\SetKwHangingKw{Func}{Function}{}{}
 \KwData{$q$;}
 \KwResult{$\mathcal{A}^*(q)$;}
    \For {each substring $q_i$ of $q$, with length $|q_i|$ from 1..$|q|$ \label{line:1}}{
        $P(\mathcal{A}^*(q_i))\leftarrow\delta(q_i)$ \label{line:bpp:5} \;
         \If {$\delta(q_i)=1$} {$\mathcal{A}^*(q_i)\leftarrow \{q_i\}$; } \label{line:bpp:2}
        \For {each substring $q_j$ of $q_i$ \label{line:2}}{
            $r(q_i,q_j) \leftarrow$ Replace $q_j$ in $q_i$ with $``\$e"$ \; \label{line:3}
            \If {$P(\mathcal{A}^*(q_i))<P(r(q_i,q_j))P(\mathcal{A}^*(q_j))$ \label{line:4}}{
                $\mathcal{A}^*(q_i)\leftarrow \mathcal{A}^*(q_j)\oplus r(q_i, q_j)$ \;
                $P(\mathcal{A}^*(q_i))\leftarrow P(r(q_i,q_j))P(\mathcal{A}^*(q_j))$ \label{line:bpp:6} \;
            }
        }
    }
    \Return $\mathcal{A}^*(q)$
 \caption{\bf \scriptsize Complex Question Decomposition}
\label{alg:bpp}
\end{algorithm}

\vspace{-0.3cm}
\section{Predicate Expansion}
\label{sec:expansion}
In a knowledge base, many facts are not expressed by a direct predicate, but by a path consisting of multiple predicates. As shown in Figure~\ref{fig:barackobama}, ``spouse of''  relationship is represented by three predicates $marriage \rightarrow person \rightarrow name$. We denote these multi-predicate paths as \emph{expanded predicates}. Answering questions over expanded predicates highly improves the coverage of KBQA.

\begin{definition}[Expanded Predicate]
\label{def:ep}
An expanded predicate $p^+$ is a predicate sequence $p^+ = (p_1, ..., p_k)$. We refer to
$k$ as the {\it length} of the $p^+$. We say $p^+$ connects subject $s$ and object $o$, if there exists a sequence of subjects $s=(s_1,s_2,...,s_k)$ such that
$\forall 1 \le i <k, (s_i,p_i,s_{i+1}) \in \mathcal{K}$ and $(s_k,p_k,o) \in \mathcal{K}$. Similar to $(s,p,o) \in  \mathcal{K}$ representing $p$ connects $s$ and $o$, we denote $p^+$ connecting $s$ and $o$ as $(s,p^+,o) \in \mathcal{K}$.
\end{definition}

The KBQA model proposed in Section 3 in general is flexible for expanded predicates. We only need some slight changes for the adaptation. In Sec~\ref{sec:expansion:learn}, we show such adaptation. Then we show how to scale expanded predicates for billion scale knowledge bases in Sec~\ref{sec:expansion:generation}. Finally, we show how to select a reasonable predicate length to pursue highest effectiveness in Sec~\ref{sec:expansion:selection}.


\subsection{KBQA for Expanded Predicates}
\label{sec:expansion:learn}
Recall that the framework of KBQA for single predicate consists of two major parts. In the offline part, we compute $P(p|t)$, the predicate distribution for given templates; in the online part, we extract the question's template $t$, and compute the predicate through $P(p|t)$.
When changing $p$ to $p^{+}$, we do the following adjustments:

In the {\bf offline} part, we learn question templates for expanded predicates, i.e., computing $P(p^+|t)$.
The computation of $P(p^+|t)$ only relies on knowing whether $(e,p^+,v)$ is in $\mathcal{K}$. We can compute this cardinality if we generate all $(e,p^+,v) \in \mathcal{K}$. We show this generation process in Sec~\ref{sec:expansion:generation}

In the {\bf online} part, we use expanded predicates to answer question. To compute $P(v|e,p^+)$, we can compute it by exploring the RDF knowledge base starting from $e$ and going through $p^+$. For example, let $p^{+}=marriage \rightarrow person \rightarrow name$, to compute $P(v|Barack\ Obama,p^+)$ from the knowledge base in Figure~\ref{fig:barackobama}, we start the traverse from node $a$, then go through $b$, $c$. Finally we have $P(Michelle\ Obama|Barack\ Obama, p^+)=1$.


\subsection{Generation of Expanded Predicates}
\label{sec:expansion:generation}
A naive approach to generate all expanded predicates is breadth-first search (BFS) starting from each node in the knowledge base. However, the number of expanded predicates grows exponentially with the predicates' length. So the cost is unacceptable for a billion scale knowledge base. 
To do this, we first set a limit on the predicate length $k$ to improve the scalability. That is we only search expanded predicate with length no larger than $k$. In the next subsection, we will show how to set a reasonable $k$. In this subsection, we improve the scalability from another two aspects: (1) reduction on $s$; (2) memory-efficient BFS.

\textbf{Reduction on $s$}: During the offline inference process, we are only interested in $s$ which occurred in at least one question in the QA corpus. Hence, {\it we only use subjects occurring in the questions from QA corpus as starting nodes for the BFS exploration.} This strategy significantly reduces the number of $(s,p^+,o)$ triples to be generated. Because the number of such entities is far less than the number of those in a billion-scale knowledge base. For the knowledge base (1.5 billion entities) and QA corpus (0.79 million distinct entities) we use, this filtering reduces the number of $(s,p^+,o)$ triples $1500/0.79=1899$ times theoretically.

\textbf{Memory-Efficient BFS}:
To enable the BFS on a knowledge base of 1.1TB, we
use a {\it disk based multi-source BFS} algorithm.
At the very beginning, we load all entities occurring in QA corpus (denoted by $S_0$) into memory and build the hash index on $S_0$. In the first round, by scanning all RDF triples resident on disk once and joining the subjects of triples with $S_0$, we get all $(s,p^+,o)$ with length 1. The hash index built upon $S_0$ allows a linear time joining. In the second round, we load all the triples found so far into memory and build hash index on all objects $o$ (denoted by $S_1$). Then we scan the RDF again and join the subject of RDF tripes with $s \in S_1$. Now we get all $(s,p^+,o)$ with length 2, and load them into memory. We repeat the above index+scan+join operation $k$ times to get all $(s,p^+,o)$ with $p^+.length \le k$.

The algorithm is efficient since our time cost is mainly spent on scanning the knowledge base $k$ times.
The index building and join are executed in memory, and the time cost is negligible compared to the disk I/O. Note that the number of expanded predicate starting from $S_0$ is always significantly smaller than the size of knowledge base, thus can be hold in memory. For the knowledge base (KBA, please refer to experiment section for more details) and QA corpus we use, we only need to store 21M $(s,p^+,o)$ triples. So it's easy to load them into memory. Suppose the size of $\mathcal{K}$ is $|\mathcal{K}|$, and the number of $(s,p^+,o)$ triples found is $\#spo$. It consumes $O(\#spo)$ memory, and the time complexity is $O(|\mathcal{K}|+\#spo)$.

\subsection{Selection of $k$}
\label{sec:expansion:selection}

The length limit $k$ of expanded predicate affects the {\it
effectiveness} of predicate expansion. A larger $k$
leads to more $(s,p^+,o)$ triples, and consequently higher {\it coverage} of
questions. However, it also introduces more meaningless $(s,p^+,o)$ triples. For example, the expanded predicate $marriage \rightarrow person \rightarrow dob$ in Figure~\ref{fig:barackobama} connects ``Barack Obama'' and ``1964''. But they have no obvious relations and are useless for KBQA.

The predicate expansion should select a $k$ that allows most meaningful relations and avoids most meaningless relations. We estimate the best $k$ using Infobox in Wikipedia. Infobox stores facts about entities and most entries in Infobox are subject-predicate-object triples. Facts in Infobox can be used as meaningful relations.
Hence, our idea is  {\it sampling $(s,p^+,o)$ triples with length $k$ and see how many of them have correspondence in Infobox}. We expect to see a significant drop for an excessive $k$.





Specifically, we select top 17,000 entities from the RDF knowledge base $\mathcal{K}$
ordered by their frequencies. The frequency of an entity $e$ is
defined as the number of $(s,p,o)$ triples in $\mathcal{K}$ so that $e=s$. We choose
these entities because they have richer facts, and therefore are more trustworthy. For these entities, we generate their $(s,p^+,o)$ triples at length $k$ using the BFS procedure proposed in Sec~\ref{sec:expansion:generation}. Then, for each $k$, we count the number of these $(s,p^+,o)$ that can find its corresponding SPO triples in Wikipedia Infobox.
More formally, let $E$ be the sampled entity set, and $SPO_k$ be $(s,P^+,o) \in \mathcal{K}$ with length $k$. We define $valid(k)$ to measure the influence of $k$ in finding meaningful relations as follows:
\begin{flalign}
\scriptsize
\begin{split}
\label{eqn:valid}
valid(k) =\sum_{s\in E} |\{(s,p^+,o)|(s,p^+,o) \in SPO_k, \\ \exists p,  (s,p,o) \in Infobox \}|
\end{split}
\end{flalign}
The results of $valid(k)$ over KBA and DBpedia are shown in Table~\ref{tab:wikidistribution}. Note that for expanded predicates with length $\ge 2$, we only consider those which end with $name$. This is because we found entities and values connected by other expanded predicates always have some very weak relations and should be discarded. The number of valid expanded predicates significantly drops when $k=3$. This suggests that most meaningful facts are represented within this length. So we choose $k=3$ in this paper.

\begin{table}[!htb]
\vspace{-0.2cm}
\scriptsize
\begin{center}
\caption{\bf valid(k)}
\begin{tabular}{   c | c | c | c }
  \hline
     k & 1 & 2 & 3 \\ \hline
     \hline
     KBA     & 14005 & 16028 & 2438 \\ \hline
     DBpedia & 352811 & 496964 & 2364 \\ \hline
\end{tabular}
\label{tab:wikidistribution}
\end{center}
\vspace{-1cm}
\end{table}

\section{Experiments}
\label{sec:exp}

In this section, we first elaborate the experimental setup in Sec~\ref{sec:exp:setup}, and verify the rationality of our probabilistic framework in Sec~\ref{sec:rationality}. We evaluate the effectiveness and efficiency of our system in Sec~\ref{sec:exp:effectiveness} and Sec~\ref{sec:exp:efficiency}, respectively. In Sec~\ref{sec:exp:components}, we also investigate the effectiveness of three detailed components of KBQA. 

\subsection{Experimental Setup}
\label{sec:exp:setup}

We ran KBQA on a computer with Intel Xeon CPU, 2.67 GHz, 2 processors, 24 cores, 96 GB memory, 64 bit windows server 2008 R2. We use Trinity.RDF~\cite{zeng2013distributed} as the RDF engine. 

{\bf Knowledge base.} We use three open domain RDF knowledge bases. The business privacy rule prohibits us from publishing the name of the first knowledge base. We refer to it as KBA. KBA contains 1.5 billion entities and 11.5 billion SPO triples, occupying 1.1 TB storage space. The SPO triples cover $2658$ distinct predicates and $1003$ different categories. For the sake of reproducibility, we also evaluated our system on two well-known public knowledge bases Freebase and DBpedia. Freebase contains 116 million entities and 2.9 billion SPO triples, occupying 380 GB storage. DBpedia contains 5.6 million entities, 111 million SPO triples, occupying 14.2 GB storage.


{\bf QA corpus.} The QA corpus contains 41 million QA pairs crawled from Yahoo! Answer. If there are multiple answers for a certain question, we only consider the {\it ``best answer''}.

{\bf Test data.} We evaluated KBQA over WebQuestions~\cite{berant2013semantic}, QALD-5~\cite{unger2015qald}, QALD-3~\cite{unger2013qald} and QALD-1~\cite{ungerqald}, which are designed for QA systems over {\it knowledge bases}.
We present the basic information of these data sets in Table~\ref{tab:testdata}. We are especially interested in the number of questions which can be decomposed into BFQs ($\#BFQ$) since KBQA focuses on answering BFQs.

{
\setlength{\tabcolsep}{2.5pt}
\begin{table}[!htb]
\vspace{-0.2cm}
\begin{center}
\small
\caption{Benchmarks for evaluation.}
\begin{tabular}{  l  | l | l | l || l  | l | l | l  }
\hline
 & \#total  & \#BFQ & ratio & &\#total  & \#BFQ & ratio\\ \hline
 \hline
WebQuestions & 2032 & - & - & QALD-3  & 99 & 41 & 0.41\\ \hline
QALD-5 & 50 & 12  & 0.24 & QALD-1 & 50 & 27 & 0.54\\ \hline
\end{tabular}
\label{tab:testdata}
\end{center}
\vspace{-0.8cm}
\end{table}
}

\subsection{Rationality of Probabilistic Framework}
\label{sec:rationality}
We explain why a probabilistic framework is necessary. In each step of the question understanding, there are different choices which bring uncertainty to our decision. We show the number of candidate choices in each step over KBA in Table~\ref{tab:rationality}. Such uncertainty suggests a probabilistic framework.


\begin{table}[!htb]
\vspace{-0.2cm}
\small
\begin{center}
\caption{\bf Statistics for average choices of each random variable.}
\begin{tabular}{  l | l | l }
\hline
  Probability & Explanation & Avg. Count  \\ \hline
  \hline
  $P(e|q)$ & \#entity for a question &  18.7 \\ \hline
  $P(t|e,q)$ & \#templates for a entity-question pair & 2.3 \\ \hline
  $P(p|t)$ & \#predicates for a template &  119.0 \\ \hline
  $P(v|e,p)$ & \#values for a entity-predicate pair & 3.69 \\ \hline
\end{tabular}
\label{tab:rationality}
\end{center}
\vspace{-0.7cm}
\end{table}

As an example, $P(t|e,q)$ is used to represent the uncertainty when translating a question and its entity into a template. For the question {\tt How long is Mississippi River?}, after identifying {\tt Mississippi River}, we still cannot decide the unique category from many candidates, such as RIVER, LOCATION.

\subsection{Effectiveness}
\label{sec:exp:effectiveness}
To evaluate the effectiveness of KBQA, we conduct the following experiments. For the online part, we evaluate the {\it precision} and {\it recall} of question answering. For the offline part, we evaluate the {\it coverage} and {\it precision} of predicate inference.

\subsubsection{Effectiveness of Question Answering}

{\bf Metric for QALD}
A QA system may return null when it believes that there is no answer. So we are interested in the number of questions a QA system processed and returned a non-null answer (not necessarily true) ($\#pro$), and the number of questions whose answers are right ($\#ri$).
However, in reality, the system can only partially correctly answer a question (for example, only finding a part of the correct answers).
Hence we also report the number of questions whose answers are partially right ($\#par$).
Next we define these metrics for KBQA. Once a predicate is found, the answer can be trivially found from RDF knowledge base. Hence, for KBQA, $\#pro$ is the number of questions that KBQA finds a predicate. $\#ri$ is the number of questions for which KBQA finds right predicates. $\#par$ is the number of questions for which KBQA finds partially right predicates.
For example, ``place of birth'' is a partially correct predicate for {\tt Which city was \$person born?}, as it may refer to a country or a village instead of a city.

Now we are ready to define our metrics:  {\it precision} $P$,  {\it partial precision} $P^*$, {\it recall} $R$ and {\it partial recall} $R^*$ as follows:
{
\scriptsize
\begin{flalign*}
\begin{split}
P=\frac{\#ri}{\#pro}; P^*=\frac{\#ri+\#par}{\#pro}; R=\frac{\#ri}{\#total}; R^*=\frac{\#ri+\#par}{\#total}
\end{split}
\end{flalign*}
}
We are also interested in the recall and partial recall with respect to the number of BFQs, denoted by $R_{BFQ}$ and $R^*_{BFQ}$:
{
\begin{flalign*}
\scriptsize
\begin{split}
R_{BFQ}=\frac{\#ri}{\#BFQ}; R^*_{BFQ}=\frac{\#ri+\#par}{\#BFQ}
\end{split}
\end{flalign*}
\vspace{-0.3cm}
}


{\bf Results on QALD-5 and QALD-3} We give the results in Table~\ref{tab:qald5} and Table~\ref{tab:qald3}.
For all the competitors, we directly report their results in their papers.
We found that on all knowledge bases, KBQA beats all other competitors except squall2sparql in terms of \emph{precision}.
This is because squall2sparql employs humans to identify the entity and predicate for each question.
We also found that KBQA performs the best over DBpedia than other knowledge bases. This is because the QALD benchmark is mainly designed for DBpedia. For most questions in QALD that KBQA can process KBQA can find the right answers from DBpedia.

\begin{table}[!htb]
\vspace{-0.2cm}
\scriptsize
\begin{center}
\caption{\small Results on QALD-5.}
\begin{tabular}{  p{1.6cm} | p{0.3cm} | l | p{0.3cm} | p{0.25cm}  p{0.45cm} | p{0.25cm}  p{0.45cm} | p{0.3cm} | p{0.3cm} }
\hline
 & \#pro & \#ri & \#par & R & & R$^*$ & & P & P$^*$  \\ \hline
 \hline
Xser & 42 & 26 & 7 & 0.52 & & 0.66 & & 0.62 & 0.79 \\ \hline
APEQ & 26 & 8  & 5 & 0.16 & & 0.26 & & 0.31 & 0.50  \\ \hline
QAnswer & 37 & 9 & 4  & 0.18 & & 0.26 & & 0.24 & 0.35 \\ \hline
SemGraphQA    & 31 & 7 & 3 & 0.14 & & 0.20 & & 0.23 & 0.32 \\ \hline
YodaQA        & 33 & 8 & 2 & 0.16 & & 0.20 & & 0.24 & 0.30 \\ \hline
 \hline
 \multicolumn{4}{c|}{}  & R & R$_{\text{BFQ}}$ & R$^*$ & R$^*_{\text{BFQ}}$ &  \multicolumn{2}{|c}{}  \\ \hline
KBQA+KBA              & 7 & 5 & 1  & 0.10 & 0.42 & 0.12 & 0.50 & 0.71 & 0.86 \\ \hline
KBQA+Freebase         & 6 & 5 & 1  & 0.10 & 0.42 & 0.12 & 0.50 & 0.83 & 1.00 \\ \hline
KBQA+DBpedia          & 8 & 8 & 0  & 0.16 & 0.67 & 0.16 & 0.67 & \textbf{1.00} & \textbf{1.00} \\ \hline
\end{tabular}
\label{tab:qald5}
\end{center}
\vspace{-0.6cm}
\end{table}

{
\setlength{\tabcolsep}{2.5pt}
\begin{table}[!htb]
\scriptsize
\begin{center}
\caption{\small Results on QALD-3.}
\begin{tabular}{  l | l | l | l | l | l | l | l | l | l | l | l}
\hline
 & \#pro & \#ri & \#par & R & R$_{\text{BFQ}}$ & R$^*$ & R$^*_{\text{BFQ}}$ & P & P$_{\text{BFQ}}$ & P$^*$ & P$^*_{\text{BFQ}}$ \\ \hline
 \hline
squall2sparql & 96 & 80 & 13 & .78 & .81 & .91 & .94 & .84 & .95 & .97 & .95 \\ \hline
SWIP          & 21 & 14 & 2  & .14 & .24 & .16 & .24 & .67 & .77 & .76 & .77 \\ \hline
CASIA         & 52 & 29 & 8  & .29 & .56 & .37 & \textbf{.61} & .56 & .79 & .71 & .86 \\ \hline
RTV           & 55 & 30 & 4  & .30 & .56 & .34 & .56 & .55 & .72 & .62 & .72 \\ \hline
gAnswer~\cite{zou2014natural}          & 76 & 32 & 11 & \textbf{.32} & .54 & \textbf{.43} & -  & .42 & .54 & .57  & - \\ \hline
Intui2        & 99 & 28 & 4  & .28 & .54 & .32 & .56 & .28 & .54 & .32 & .56 \\ \hline
Scalewelis    & 70 & 32  & 1 & \textbf{.32} & .41 & .33 & .41 & .46 & .50 & .47 & .5 \\ \hline
\hline
KBQA+KBA    & 25 & 17 & 2  & .17 & .42 & .19 & .46 & .68 & .68 & .76  & .76 \\ \hline
KBQA+FB         & 21 & 15 & 3  & .15 & .37 & .18 & .44 & .71 & .71 & .86 & .86 \\ \hline
KBQA+DBp          & 26 & 25 & 0  & .25 & \textbf{.61} & .25 & \textbf{.61} & \textbf{.96} & \textbf{.96} & \textbf{.96} & \textbf{.96}\\ \hline
\end{tabular}
\label{tab:qald3}
\end{center}
\vspace{-0.8cm}
\end{table}
}

\nop{
\begin{table}[!htb]
\scriptsize
\begin{center}
\begin{tabular}{  p{1.6cm} | p{0.3cm} | l | p{0.3cm} | p{0.25cm}  p{0.55cm} | p{0.25cm}  p{0.55cm} | p{0.3cm} | p{0.3cm} }
\hline
 & \#pro & \#ri & \#par & R & & R$^*$ & & P & P$^*$  \\ \hline
 \hline
squall2sparql & 96 & 77 & 13 & 0.78 & & 0.91 & & 0.80 & 0.94 \\ \hline
SWIP          & 21 & 14 & 2  & 0.14 & & 0.16 & & 0.67 & 0.76  \\ \hline
CASIA         & 52 & 29 & 8  & 0.29 & & 0.37 & & 0.56 & 0.71 \\ \hline
RTV           & 55 & 30 & 4  & 0.30 & & 0.34 & & 0.55 & 0.62 \\ \hline
gAnswer          & 76 & 32 & 11 & 0.32 & & 0.43 & & 0.42 & 0.57 \\ \hline
Intui2        & 99 & 28 & 4  & 0.28 & & 0.32 & & 0.28 & 0.32 \\ \hline
Scalewelis    & 70 & 1  & 38 & 0.01 & & 0.39 & & 0.01 & 0.56 \\ \hline
 \hline
 \multicolumn{4}{c|}{}  & R & R$_{BFQ}$ & R$^*$ & R$^*_{BFQ}$ &  \multicolumn{2}{|c}{}  \\ \hline

KBQA+KBA    & 25 & 17 & 2  & 0.17 & 0.42 & 0.19 & 0.46 & \textbf{0.68} & \textbf{0.76} \\ \hline
KBQA+Freebase         & 21 & 15 & 3  & 0.15 & 0.37 & 0.18 & 0.44 & \textbf{0.71} & \textbf{0.86} \\ \hline
KBQA+DBpedia          & 26 & 25 & 0  & 0.25 & 0.61 & 0.25 & 0.61 & \textbf{0.96} & \textbf{0.96} \\ \hline
\end{tabular}
\caption{\small Results on QALD-3.}
\label{tab:qald3}
\end{center}
\vspace{-0.5cm}
\end{table}
}

{\bf Recall Analysis} The results in Table~\ref{tab:qald5} and Table~\ref{tab:qald3} imply that KBQA has a relatively low recall. The major reason is, KBQA only answer BFQs (binary factoid questions), while the QALD benchmarks contain many non-BFQs. If we only consider BFQs, the recalls increases to 0.67, and 0.61, respectively (over DBpedia). Furthermore, we studied the cases when KBQA failed to answer a BFQ on QALD-3. Many cases fail because KBQA uses a relatively strict rule for template matching.
The failure case usually happens when a rare predicate is matched against a rare question template. 12 out of 15 questions fail due to this reason. One question example is {\tt In which military conflicts did Lawrence of Arabia participate}, whose predicate in DBpedia is {\it battle}. KBQA lacks training corpora to understand these rare predicates.



{\bf Results on QALD-1} We compared with DEANNA over BFQs in QALD-1 in Table~\ref{tab:quality}. DEANNA is a state-of-the-art synonym based approach, which also focuses on BFQs.
For DEANNA, $\#pro$ is the number of questions that are transformed into SPARQL queries.
The results show that the precision of KBQA is significantly higher than that of DEANNA. Since DEANNA is a typical synonym based approach, this verifies that template based approach is superior to the synonym based approach in terms of precision.

\begin{table}[!htb]
\vspace{-0.1cm}
\begin{center}
\scriptsize
\caption{\small Results on QALD-1.}
\begin{tabular}{  l  | p{0.3cm} | p{0.3cm} | p{0.3cm} | c |c  | c | c }
\hline
 & \#pro & \#ri & \#par & R$_{BFQ}$ & R$^*_{BFQ}$ &  P & P$^*$ \\ \hline
 \hline
DEANNA & 20 & 10 & 0 & 0.37 & 0.37 & 0.5 & 0.5\\ \hline
KBQA + KBA & 13 & \textbf{12} & 0 & \textbf{0.48} & \textbf{0.48} & \textbf{0.92} & \textbf{0.92} \\ \hline
KBQA + Freebase & 14 & \textbf{13} & 0 & \textbf{0.52} & \textbf{0.52} & \textbf{0.93} & \textbf{0.92} \\ \hline
KBQA + DBpedia & 20 & \textbf{18} & 1 & \textbf{0.67} & \textbf{0.70} & \textbf{0.90} & \textbf{0.95} \\ \hline
\end{tabular}
\label{tab:quality}
\end{center}
\vspace{-0.5cm}
\end{table}

{\bf Results on WEBQUESTIONS}
We show the results of KBQA on ``WebQuestions'' in Table~\ref{tab:webquestions}. We compared KBQA with some state-of-the-art systems by precision@1 (p@1), precision (P), recall (R), and F1 score as computed by the official evaluation script. The ``WebQuestions'' data set still has many non-BFQs. The results again show that KBQA is excellent at BFQ (the high precision for BFQs). The lower F1 score is caused by the recall, that KBQA cannot answer non-BFQs.

\begin{table}[!htb]
\small
\vspace{-0.3cm}
\begin{center}
\caption{\small Results on the WEBQUESTIONS test set.}
\begin{tabular}{l|l|l|l|l}
\hline
system & P & P@1 & R & F1 \\
\hline
\hline
(Bordes et al., 2014) \cite{DBLP:conf/emnlp/BordesCW14} & - & 0.40 & - & 0.39 \\ \hline
(Zheng et al., 2015) \cite{zheng2015build} & 0.38 & - &  - & - \\ \hline
(Li et al., 2015) \cite{dong2015question} & - & 0.45 & - & 0.41  \\ \hline
(Yao, 2015) \cite{yao2015lean} & 0.53 & - & 0.55 & \textbf{0.44} \\ \hline
KBQA  & \textbf{0.85} & \textbf{0.52} & 0.22 & 0.34 \\ \hline
\end{tabular}
\label{tab:webquestions}
\end{center}
\vspace{-0.6cm}
\end{table}

{\bf Results for hybrid systems} Even for a dataset that the BFQ is not a majority (e.g. WEBQUESTIONS, QALD-3), KBQA contributes by being an excellent component for hybrid QA systems.
We build the hybrid system by: first, the user question is first fed into KBQA. If KBQA gives no reply, which means the question is very likely to be a non-BFQ, we fed the question into the baseline system. We show the improvements of the hybrid system in Table~\ref{tab:hybrid}. The performances of all baseline systems significantly improve when working with KBQA. The results verify the effectiveness of KBQA for a dataset that the BFQ is not a majority.

{
\setlength{\tabcolsep}{2.5pt}
\begin{table}[!htb]
\scriptsize
\begin{center}
\caption{\small Results of hybrid systems on QALD-3 over DBpedia.}
\begin{tabular}{l|l|l|l|l}
\hline
System & R & R* & P & P*\\
\hline
\hline
SWIP & 0.15 & 0.17 & 0.71 & 0.81 \\
KBQA+SWIP & 0.33(+0.18) & 0.35(+0.18) & 0.87(+0.16) & 0.92(+0.11) \\
\hline
CASIA & 0.29 & 0.37 & 0.56 & 0.71\\
KBQA+CASIA & 0.38(+0.09) & 0.44(+0.07) & 0.66(+0.10) & 0.76(+0.05)\\
\hline
RTV & 0.3 & 0.34 & 0.34 & 0.62\\
KBQA+RTV & 0.39(+0.09) & 0.42(+0.08) & 0.66(+0.32) & 0.71(+0.09)\\
\hline
gAnswer & 0.32 & 0.43 & 0.42 & 0.57\\
KBQA+gAnswer & 0.39(+0.07) & - & - & -\\
\hline
Intui2 & 0.28 & 0.32 & 0.28 & 0.32\\
KBQA+Intui2 & 0.39(+0.11) & 0.41(+0.09) & 0.39(+0.11) & 0.41(+0.09)\\
\hline
Scalewelis & 0.32 & 0.33 & 0.46 & 0.47\\
KBQA+Scalewelis & 0.44(+0.12) & 0.45(+0.12) & 0.60(+0.14) & 0.62(+0.15)\\
\hline
\end{tabular}
\label{tab:hybrid}
\end{center}
\vspace{-0.8cm}
\end{table}
}

\subsubsection{Effectiveness of Predicate Inference}
Next we justify the effectiveness of KBQA in predicate inference by showing that (1) KBQA learns rich templates and predicates for natural language questions (\emph{coverage}), and (2) KBQA infers the correct predicate for most templates (\emph{precision}).

{\bf Coverage}
We show the number of predicates and templates KBQA learns in Table~\ref{tab:templatecount}, with comparison to {\it bootstrapping}~\cite{yahya2012natural,unger2012template}, which is a state-of-the-art synonym based approach. Bootstrapping learns synonyms (BOA patterns, which mainly contains text between subjects and objects in web docs) for predicates from knowledge base and web docs. The BOA patterns can be seen as templates, and the relations can be seen as predicates.

{
\setlength{\tabcolsep}{2.5pt}
\begin{table}[!htb]
\vspace{-0.2cm}
\scriptsize
\begin{center}
\caption{\small Coverage of Predicate Inference}
\begin{tabular}{  l | c | c | c | c}
\hline
                        & KBQA & KBQA & KBQA  & Bootstrapping\\
                        & +KBA & +Freebase & +DBpedia & \\ \hline
\hline
  Corpus & 41M QA & 41M QA & 41M QA & 256M sentences \\ \hline
  Templates                     & \textbf{27126355}  & 1171303  & 862758 &   471920        \\ \hline
  Predicates   & \textbf{2782} & 4690 & 1434     & 283              \\ \hline
  Templates per predicate  & \textbf{9751}  &250 & 602 & 4639 \\  \hline

\end{tabular}
\label{tab:templatecount}
\end{center}
\vspace{-0.5cm}
\end{table}
}

The results show that KBQA finds significantly more templates and predicates than its competitors despite that bootstrapping uses larger corpus. This implies that KBQA is more effective in predicate inference: (1) the large number of templates ensures that KBQA understands diverse question templates; (2) the large number of predicates ensures that KBQA understands diverse relations. Since KBA is the largest knowledge base we use, and KBQA over KBA generates the most templates, we focus on evaluating KBA in the following experiments.

{\bf Precision}
Our goal is to evaluate whether KBQA generates a correct predicate for a given template.
For this purpose, we select the top 100 templates ordered by their frequencies. We also randomly select 100 templates with frequency larger than 1 (templates only occur once usually have very vague meanings). For each template $t$, we manually check whether the predicate $p$ with maximum $P(p|t)$ is correct.  Similar to the evaluation on QALD-3, in some cases the predicate is \emph{partially} right.
The results are shown in Table~\ref{tab:matchingrate}. On both template sets, KBQA achieves high precision. The precision of the top 100 templates even reaches 100\%. This justifies the quality of the inference from templates to predicates.

\begin{center}
\vspace{-0.2cm}
\small
\captionof{table}{\small Precision of Predicate Inference}
\begin{tabular}{ l | c | c | c | c }
\hline
Templates & $\#right$ & $\#partially$ & P & P$^*$ \\ \hline
\hline
Random 100 & 67 & 19 & 67\% & 86\% \\ \hline
Top 100 & 100 & 0 & 100\% & 100\% \\ \hline
\end{tabular}
\label{tab:matchingrate}
\vspace{-0.2cm}
\end{center}


\subsection{Efficiency}
\label{sec:exp:efficiency}
We first give the running time. The we present the time complexity analysis for KBQA.

{\bf Running Time Experiments} There two parts of the running time: offline and online. The offline procedure, which mainly learns templates, takes 1438 min. The time cost is mainly caused by the large scale of data: billions of facts in knowledge base and millions of QA pairs. Since the offline procedure only runs once, the time cost is acceptable. The online part is responsible for question answering. We present the online time cost of our solution in Table~\ref{tab:time} with the comparison to gAnswer and DEANNA.
KBQA takes 79ms, which is 13 times faster than gAnswer, and 98 times faster than DEANNA. This implies that KBQA efficiently supports real-time QA, which is expected in real applications.

\begin{table}[!htb]
\scriptsize
\begin{center}
\caption{Time cost}
\begin{tabular}{ c | c | c | c  }
  \hline
    & Time & \multicolumn{2}{c}{Time Complexity} \\ \hline
  \hline
  \multicolumn{2}{c|}{} & Question Understanding & Question Evaluation \\ \hline
 DEANNA  & 7738ms & NP-hard & NP-hard  \\ \hline
 gAnswer   & 990ms & $O(|V|^3)$ & NP-hard \\ \hline
 \multicolumn{2}{c|}{} & Question Parsing & Probabilistic Inference\\ \hline
 KBQA  & \textbf{79ms} & $O(|q|^4)$ & $O(|P|)$  \\ \hline
\end{tabular}
\label{tab:time}
\end{center}
\vspace{-0.8cm}
\end{table}

\vspace{-3mm}
\paragraph*{\bf Complexity Analysis}
We also investigate their time complexity in Table~\ref{tab:time}, where $|q|$ is the question length, and $|V|$ in gAnswer is the number of vertices in RDF graph. As can be seen,  all procedures of KBQA have polynomial time complexity, while both gAnswer and DEANNA suffer from some NP-hard procedures. The complexity of question understanding for gAnswer is $O(|V|^3)$. Such complexity is unacceptable on a billion scale knowledge base. In contrast, the complexity of KBQA is $O(|q|^4)$ and $O(|P|)$ ($|P|$ is the number of distinct predicates), which is independent of the size of the knowledge base. As shown in Sec~\ref{sec:complex:solution}, over 99\% questions' length is less than 23. Hence, KBQA has a significant advantage over its competitors in terms of time complexity.

\subsection{Detailed Components}
\label{sec:exp:components}

We investigate the effectiveness of the three specific components in KBQA: entity\&value identification (in Sec~\ref{sec:observationformalization}), complex question answering (in Sec~\ref{sec:complex}), and predicate expansion (in Sec~\ref{sec:expansion}).

{\bf Precision of Entity\&Value Identification}
Note that most previous studies focused solely on entity extraction and cannot be used to extract entity and value simultaneously. Hence, we can only compare to a state-of-the-art solution for entity identification, Stanford Named Entity Recognizer~\cite{finkel2005incorporating}.
We randomly select 50 QA pairs whose answers are covered by the knowledge base.
We manually check whether the extracted entity is correct. Our approach correctly identifies entities for
36 QA pairs (72\%), which is superior to Stanford NER that identifies entities correctly for only 15 QA pairs (30\%).
This result suggests that {\it joint extraction of entities is better than the independent extraction}.


\nop{
\begin{table}[!htb]
\small
\begin{center}
\begin{tabular}{  l | c | c }
\hline
 & KBQA & Stanford NER\\ \hline
\hline
$\#right$ & \textbf{36} & 15 \\ \hline
Accuracy & \textbf{72\%} & 30\% \\ \hline
\end{tabular}
\caption{Entity Detection Results}
\label{tab:entitydetection}
\end{center}
\vspace{-0.5cm}
\end{table}
}

{\bf Effectiveness to Answer Complex Questions}
Since no benchmark is available for complex question answering, we constructed 8 such questions as shown in Table~\ref{tab:cq}. All these questions are typical complex questions posed by real users. We compare KBQA with two state-of-the-art QA engines: Wolfram Alpha and gAnswer. Table~\ref{tab:cq} shows the result. We found that KBQA beats the strong competitors in answering complex questions. This implies that KBQA is effective in answering complex questions.

\begin{center}
\scriptsize
\captionof{table}{\small Complex Question Answering. WA stands for Wolfram Alpha, and gA stands for gAnswer.}
\begin{tabular}{ l | c | c | c}
\hline
Question & KBQA & WA & gA \\ \hline
\hline
How many people live in the capital of Japan? & Y & Y & N \\ \hline
When was Barack Obama's wife born? & Y & Y & N \\ \hline
What are books written by author of Harry Potter? & Y & N & N \\ \hline
What is the area of the capital of Britain? & Y & N & N\\ \hline
How large is the capital of Germany? & Y & N & N \\ \hline
What instrument do members of Coldplay play? & Y & N & N \\ \hline
What is the birthday of the CEO of Google? & Y & N & N \\ \hline
In which country is the headquarter of Google located? & Y & N & N \\ \hline
\end{tabular}
\label{tab:cq}
\vspace{-0.15cm}
\end{center}

\nop{
\begin{center}
\vspace{-0.3cm}
\small
\begin{tabular}{ l | c | c | c}
\hline
Question & KBQA & WA & G \\ \hline
\hline
How many people live in the capital of Japan? & Y & Y & Y \\ \hline
When was Barack Obama's wife born? & Y & Y & Y \\ \hline
What are books written by author of Harry Potter? & Y & N & N \\ \hline
What is the area of the capital of Britain? & Y & N & Y\\ \hline
How large is the capital of German? & Y & N & Y \\ \hline
\end{tabular}
\captionof{table}{Complex question answering. WA stands for Wolfram Alpha, and G stands for Google.}
\label{tab:cq}
\end{center}
}

{\bf Effectiveness of Predicate Expansion}
Next we show that our predicate expansion procedure is effective in two aspects.
First, the expansion can {\it find significantly more predicates}. Second, the expanded predicates enable KBQA to {\it learn more templates}. We present the evaluation results in Table~\ref{tab:effectivenesspe}. We found that (1) the expansion (with length varying from 2 to $k$) generates ten times the number of direct predicates (with length 1), and (2) with the expanded predicates, the number of templates increases by 57 times.

We further use two case studies to show (1) {\it the expanded predicates are meaningful} and (2) {\it the expanded predicates are correct}. We list 5 expanded predicates we learned in Table~\ref{tab:eep}. We found that
all these expanded predicates found by KBQA are meaningful.  We further choose one expanded predicate, $marriage \rightarrow person$ $\rightarrow name$, to see whether the templates learned for this predicate are correct or meaningful. We list five learned templates in Table~\ref{tab:spouse}. These templates in general are reasonable.



\begin{table}[!htb]
\vspace{-0.1cm}
	\begin{minipage}[t]{0.24\textwidth}
        \vspace{5pt}
		\scriptsize
		\centering
        \caption{\small Effectiveness of Predicate Expansion}
        \begin{tabular}{  l | c | c }
        \hline
                     \scriptsize{Length} & \scriptsize{$\#$Template} & \scriptsize{$\#$Predicate}  \\ \hline
        \hline
          $1$   & 467,393   & 246        \\ \hline
          $2$ to $k$ & 26,658,962 & 2536     \\ \hline
          Ratio & 57.0 & 10.3 \\ \hline
        \end{tabular}
        \label{tab:effectivenesspe}
	\end{minipage}
	\begin{minipage}[t]{0.23\textwidth}
        \vspace{5pt}
		\scriptsize
		\centering
        \caption{\small Templates for $marriage \rightarrow person \rightarrow name$ }
        \begin{tabular}{  l }
        \hline
        Who is \$person marry to? \\ \hline
        Who is \$person's husband? \\ \hline
        What is \$person's wife's name?    \\ \hline
        Who is the husband of \$person?  \\ \hline
        Who is marry to \$person? \\ \hline
        \end{tabular}
        \label{tab:spouse}
	\end{minipage}
\vspace{-0.5cm}
\end{table}



\begin{table}[!htb]
\small
\vspace{-0.1cm}
\begin{center}
\caption{Examples of Expanded Predicates}
\begin{tabular}{ l | l }
\hline
Expanded predicate & Semantic \\ \hline
\hline
marriage $\rightarrow$ person $\rightarrow$ name & spouse \\ \hline
organization\_members $\rightarrow$ member $\rightarrow$ alias & organization's member\\ \hline
nutrition\_fact $\rightarrow$ nutrient $\rightarrow$ alias & nutritional value \\ \hline
group\_member $\rightarrow$ member $\rightarrow$ name & group's member \\ \hline
songs $\rightarrow$ musical\_game\_song $\rightarrow$ name & songs of a game \\ \hline
\end{tabular}
\label{tab:eep}
\end{center}
\vspace{-0.9cm}
\end{table}

\section{Related Works}
\label{sec:related}


{\bf Natural Language Documents vs Knowledge Base}
QA is very dependent on the quality of corpora. Traditional QA systems use web docs or Wikipedia as the corpora to answer questions. State-of-the-art methods in this category~\cite{ravichandran2002learning,kwok2001scaling,dang2007overview,hirschman2001natural} usually take the sentences from the web doc or Wiki as candidate answers, and rank them based on the relatedness of words between questions and candidate answers. They also tend to use noise reduction methods such as question classification~\cite{metzler2005analysis,zhang2003question} to increase the answer's quality. In recent years, the emergence of many large scale knowledge base, such as Google Knowledge Graph, Freebase~\cite{bollacker2008freebase}, and YAGO2\cite{hoffart2011yago2}, provide a new opportunity to build a better QA system~\cite{ou2008automatic,unger2011pythia,unger2012template,gerber2011bootstrapping,yahya2012natural,ferrucci2010building}. The knowledge base in general has a more structured organization and contains more clear and reliable answers compared to the free text based QA system.


{\bf QA Systems Running on Knowledge Base}
The core process of QA systems built upon knowledge base is the predicate identification for questions. For example, the question can be answered if we can  find the predicate ``population'' from question {\tt How many people are there in Honolulu}. The development of these knowledge bases experienced three major stages: \emph{rule based}, \emph{keyword based}, and \emph{synonym based} according to predicate identification approach.
Rule based approaches map questions to predicates using manually constructed rules. For example, Ou et al.~\cite{ou2008automatic} think the question in the form of {\tt What is the <xxx> of entity?} should be mapped to the predicate <xxx>. Manually constructed rules always have high precision but low recall.
Keyword based methods~\cite{unger2011pythia} use keywords or phrases in the questions as features to find the mappings between questions and predicates. But in general, it is difficult to use keywords to find mappings between questions and
complicated predicates. For example, it is hard to map question
  {\tt how many people are there in ...?} to the predicate
  ``population'' based on keywords such as ``how many'', ``people'',
  ``are there'', etc.
Synonym based approaches~\cite{unger2012template,yahya2012natural} extend keyword based methods by taking synonyms of the predicates into consideration.  This enables to answer more questions. The key factor of the approach is the quality of synonyms. Unger et al.~\cite{unger2012template} uses the bootstrapping~\cite{gerber2011bootstrapping} from web docs to generate synonyms. Yahya et al.~\cite{yahya2012natural} generates synonyms from Wikipedia. However, synonym based approaches still cannot answer complicated questions due to the same reason as the key word based approach.
True knowledge~\cite{tunstall2010true} uses key words/phrases to represent a template. In contrast, our template is a question with its entity replaced by its concept. We think True knowledge should be categorized as a synonym-based approach. To pursue a high precision, the CANaLI system~\cite{mazzeoanswering} guided users to ask questions with given patterns. However, it only answers questions with pre-defined patterns.

In general, previous QA systems over knowledge base still have certain weakness in precision or recall, since they cannot understand the complete question.

\vspace{-0.3cm}
\section{Conclusion}
\balance
\label{sec:conslusion}
QA over knowledge bases now becomes an important and feasible task. In this paper, we build a question answering system KBQA over a very large open domain RDF knowledge base.
Our QA system distinguishes itself from previous systems in four aspects: (1) understanding questions with templates, (2) using template extraction to learn the mapping from templates and predicates, (3) using expanded predicates in RDF to expand the coverage of knowledge base, (4) understanding complex questions to expand the coverage of the questions. The experiments show that KBQA is effective and efficient, and beats state-of-the-art competitors, especially in terms of precision.

{\bf Acknowledgement} This paper was supported by the National
Key Basic Research Program of China under No.2015CB358800,
by the National NSFC (No.61472085, U1509213), by Shanghai
Municipal Science and Technology Commission foundation key
project under No.15JC1400900, by Shanghai Municipal Science
and Technology project under No.16511102102. Seung-won Hwang was supported by IITP grant funded by the Korea government (MSIP; No. B0101-16-0307) and Microsoft Research.


{\small
\bibliographystyle{abbrv}
\bibliography{qa2}}




\end{document}